\documentclass[twoside,11pt]{article}

\usepackage{jmlr2e}

\usepackage[utf8]{inputenc}
\usepackage[T1]{fontenc}
\usepackage{amsmath}
\usepackage{booktabs}
\usepackage{nicefrac}
\usepackage{microtype}
\usepackage[dvipsnames]{xcolor}
\usepackage{caption}
\usepackage{subcaption}
\usepackage{enumitem}
\usepackage{pgf}
\usepackage{tcolorbox}
\usepackage{tikz}

\newtheorem{assumption}[theorem]{Assumption}

\DeclareMathOperator{\Tr}{tr}
\DeclareMathOperator{\diag}{diag}

\DeclareMathOperator{\Cov}{Cov}

\newcommand{\E}{\mathbb{E}}

\newcommand{\R}{\mathbb{R}}

\newtcolorbox{visionbox}{colback=SpringGreen!30!White!70,colframe=ForestGreen,boxrule=0.3pt,boxsep=0pt,left=6pt,right=6pt,top=6pt,bottom=6pt}
\newcommand{\vision}[1]{\vspace{0pt}\begin{visionbox}{\textcolor{Black}{#1}}\end{visionbox}\vspace{0pt}}

\graphicspath{{figs/}}

\title{Stochastic Optimization Under Power-Law Spectra:\\
Tight Bounds and Shuffling Analysis}

\author{%
\name Christopher De Sa \email cmd353@cornell.edu \\
\addr Together.AI, Cornell University
\AND
\name Thomas Dybdahl Ahle \email thomas@normalcomputing.com \\
\addr Normal Computing
\AND
\name Yaroslav Bulatov \email yaroslavvb@mlcollective.org \\
\addr Together.AI
\AND
\name Christopher R{\'e} \email chrismre@cs.stanford.edu \\
\addr Stanford University
}

\editor{}
\makeatletter\def\@starteditor{}\makeatother

\begin{document}

\maketitle

\begin{abstract}
Recent work has established that power-law spectral conditions on data enable tight convergence bounds for deterministic gradient descent, resolving the conflict between classical exponential bounds and observed power-law learning curves. In this work, we extend this result to the stochastic regime of high-dimensional machine learning. We provide two main contributions: (1) We generalize the power-law spectral theory to Stochastic Gradient Descent (SGD), showing that the same spectral exponents govern stochastic dynamics; (2) For the fundamental case of isotropic Gaussian data, we provide a precise analysis of data shuffling, deriving exact constants that prove Single Shuffle is strictly superior to Flip-Flop and IID sampling. Our results bridge the gap between abstract spectral theory and practical stochastic training choices, offering a unified picture of how data geometry drives optimization speed.
\end{abstract}


\setcounter{tocdepth}{2}
\tableofcontents
\newpage

\begin{figure}[t]
\centering
\includegraphics[width=\linewidth]{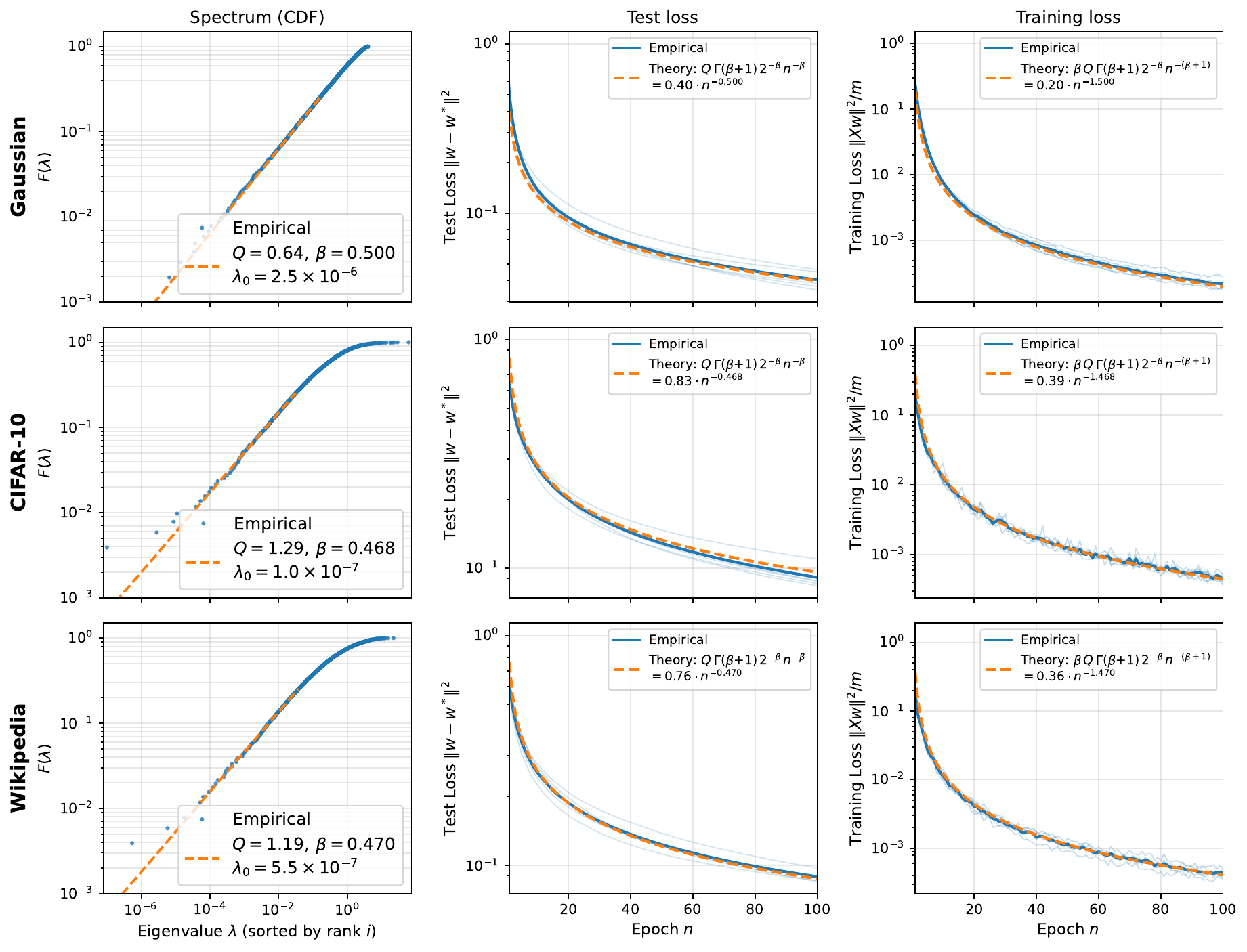}
\caption{\textbf{Validation of IID SGD theory across data sources.}
\emph{Rows:} Gaussian Data, Resnet CIFAR-10 embeddings, Transformer Wikipedia embeddings.
\emph{Columns:}
(1) empirical eigenvalue CDF of $H=X^\top X/m$ with a left-edge power-law fit $F(\lambda)\approx Q(\lambda-\lambda_0)^\beta$;
(2) test loss, $\|w-w^*\|^2$;
(3) training loss, $\|Xw\|^2/m$.
Blue curves show empirical results (5 runs); orange dashed curves show theoretical predictions using the fitted spectral constants.
Training loss curves are visibly noisier than test loss because they weight eigendirections by $\lambda_j$, amplifying fluctuations in high-eigenvalue directions.
See Section~\ref{sec:experiments} for experimental details and reproducibility information.}
\label{fig:iid-validation-intro}
\end{figure}

\section{Introduction}
\label{sec:intro}

Classical optimization theory predicts exponential convergence rates for gradient-based methods on strongly convex objectives, but in modern large-scale training we often observe approximate power-law learning curves instead \citep{kaplan2020scaling, hoffmann2022training, muennighoff2023scaling}. 
The simplest setting that exhibits this phenomenon is gradient descent on least-squares regression, with training loss function $\frac{1}{m} \| X (w - w^*) \|^2$ for data matrix $X \in \R^{m \times d}$.
Classically, let $H := \frac{1}{m}X^\top X$ denote the Gram matrix, and suppose it has condition number $\kappa(H) := \lambda_{\max}/\lambda_{\min} < \infty$.
The error $e_n := w_n - w^*$ evolves as
\( e_{n+1} = e_n - \eta H e_n \),
so $e_n = (I - \eta H)^n e_0$, and for $\eta \in (0,1/\lambda_{\max}]$,
\begin{equation}\label{eq:exponential-decay}
\|e_n\| \;\le\; (1-\eta\lambda_{\min})^n\,\|e_0\|
\le \exp(-\eta\lambda_{\min} n )\,\|e_0\|.
\end{equation}
In particular, taking $\eta = 1/\lambda_{\max}$ gives exponential decay of $\|e_n\|\le \exp(-n/\kappa(H))\|e_0\|$.

Unfortunately, for many real datasets, $\lambda_{\min}$ is very small, and so the exponential behaviour only kicks in after an impractically large number of iterations.
This makes the exponential decay of (\ref{eq:exponential-decay}) a poor predictor of convergence in these practical cases.
\citet{velikanov2024tight} addressed this by replacing $\kappa$ with a data-dependent spectral condition that captures the \emph{density} of eigenvalues near the left edge of the spectrum.
Let $\rho$ denote the cumulative distribution of eigenvalues of $H$.%
\footnote{We use $\beta$ for the spectral exponent; \citet{velikanov2024tight} use $\zeta$.
For simplicity we often write the condition as $\rho((0,\lambda]) \le Q\lambda^\beta$ when $\lambda_{\min} = 0$.}
They assume this spectral CDF satisfies a power-law near the left edge:
\begin{equation}\label{eq:spectral-condition}
\rho((\lambda_{\min},\lambda]) \;\asymp\; Q\,(\lambda - \lambda_{\min})^\beta, \qquad \lambda \text{ near } \lambda_{\min}.
\end{equation}
Under this condition, they prove that full GD with step size $\eta < 2/\lambda_{\max}$ achieves the asymptotic
\begin{equation}\label{eq:vy-result}
\frac{1}{d}\Tr\bigl((I-\eta H)^{2n}\bigr) \;=\; Q\,\Gamma(\beta+1)\,(2\eta)^{-\beta}\, n^{-\beta}\,(1 + o(1))
\quad\text{as}\quad n\to\infty,
\end{equation}
where $\Gamma$ is the Gamma function.
More precisely, this holds in the joint limit $d, n \to \infty$ with $n = o(d^{1/\beta})$.
For $n \gg d^{1/\beta}$, the exponential rate $(1 - \eta\lambda_{\min})^n$ eventually dominates,
but in modern ML $d$ is very large and so the number of steps required to reach this exponential rate regime far exceeds capabilities of modern hardware.
For moderate $d$ and shifted edges $\lambda_0>0$, we also observe a finite-$d$ effect: the smallest eigenvalue typically satisfies $\lambda_{\min}>\lambda_0$, and when $\beta\gtrsim 1$ this gap can materially change the apparent constants.\footnote{Appendix~\ref{app:powerlaw-proofs}, Proposition~\ref{prop:shifted-edge-finite-d} gives a simple ``gap+atom'' correction that improves finite-$d$ accuracy via an upper incomplete-Gamma truncation.}
Their analysis is an important step toward explaining the observed power-law convergence in practice.

\vision{In this paper we (1) extend the spectral framework of \citet{velikanov2024tight} to \emph{stochastic} optimization and (2) study the effect of different sample orders on convergence.}

\paragraph{From full gradient descent to stochastic gradient descent.}
Generalizing the results of V\&Y from full gradient descent to stochastic gradient descent is nontrivial because the dynamics matrix for SGD is a product of random, non-commuting factors, so spectral decomposition techniques do not apply.
Instead, we develop a generating-function approach to track the relevant traces through the product structure.
We prove that IID-sampled SGD achieves the sharp test-loss asymptotic under the same spectral condition, and under a mild fourth-moment closure on the spectral edge,%
\footnote{In rotationally invariant Gaussian/spherical models the needed fourth-moment closure is exact. For IID \emph{row sampling} from a fixed dataset, we prove the sharp $(1+o(1))$ asymptotic under a mild bounded-kurtosis (near-Wick) condition on the edge-whitened coordinates: if $H=V\Lambda V^\top$ and $y_i=V^\top x_i$, then $\xi_{ij}:=y_{ij}/\sqrt{\lambda_j}$ have uniformly bounded fourth moments on the left-edge window that controls the $n^{-\beta}$ regime; see Appendix~\ref{app:iid-proofs}.}
\begin{align}
\frac{1}{d}\E\Tr(S_n) &= Q\,\Gamma(\beta+1)\,(2\eta)^{-\beta}\,n^{-\beta}\,(1 + o(1)), \label{eq:sgd-test}
\end{align}
where $S_n$ is the second-moment matrix defined in Section~\ref{sec:iid}.
This matches V\&Y's result~\eqref{eq:vy-result} for full GD \emph{exactly}---stochastic sampling preserves both the exponent and the leading constant.

\paragraph{Quantifying the benefit of shuffling.}
We analyze different sampling schemes for SGD, including IID sampling and various shuffling methods.
Shuffling is widely used in practice and is known to improve convergence empirically, but prior theory has not quantified the improvement.
For isotropic Gaussian data and $d$ large, $XX^\top$ follows the Marchenko–Pastur law, which means $\beta = 1/2$, $Q = 2/\pi$.
For this regime we derive exact asymptotic expansions for the test loss with learning rate $\eta \in (0, 2)$:
\begin{align*}
\text{Shuffle-once} &= \sqrt{\frac{1/\eta - 1/2}{2\pi}}\,n^{-1/2}\left(1 + O(n^{-1/2})\right), \\
\text{IID} &= \frac{1}{\sqrt{2\pi\eta}}\,n^{-1/2}\left(1 + O(n^{-1/2})\right), \\
\text{Flip-Flop} &= \sqrt{\frac{1/\tau - 1/2}{\pi}}\,n^{-1/2}\left(1 + O(n^{-1/2})\right), \quad \tau = \eta(2-\eta).
\end{align*}
For the standard Kaczmarz step size $\eta = 1$, these simplify to $\frac{1}{2\sqrt{\pi n}}$, $\frac{1}{\sqrt{2\pi n}}$, and $\frac{1}{\sqrt{2\pi n}}$ respectively.
Shuffle-once achieves a factor of $\sqrt{2} \approx 1.41$ improvement over IID sampling at $\eta = 1$.
We conjecture this $\sqrt{2}$ ratio is universal across all spectral exponents $\beta$.

\begin{figure}[t]
\centering
\resizebox{\linewidth}{!}{\input{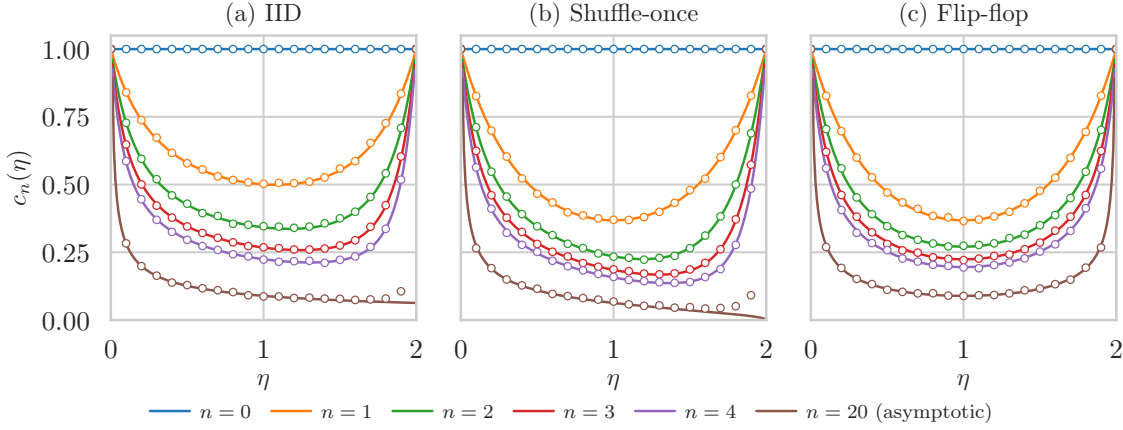}}
\caption{\textbf{Test loss $c_n(\eta)$ vs learning rate $\eta \in (0,2)$ for $d=512$, $\rho=1$.}
Solid curves show theory; hollow circles show Monte Carlo validation ($40$ probes).
(a)~IID uses ODE-based theory from the $d\to\infty$ limit.
(b)~Shuffle-once and (c)~flip-flop have exact finite-$n$ formulas; for $n{=}20$ we use the asymptotic formula.
At $\eta=1$, shuffle-once achieves a $\sqrt{2}$ factor improvement over IID.
See Appendix~\ref{app:lr} for derivations.}
\label{fig:lr-dependence-intro}
\end{figure}

\paragraph{Summary of Results: Test and Training Loss.}
We track two complementary quantities for the error: the population \emph{test loss} $\|w_n-w^*\|^2$ (the expected loss on an isotropic random example) and the sampled \emph{training loss} $\frac{1}{m}\|X(w_n-w^*)\|^2=(w_n-w^*)^\top H(w_n-w^*)$ on the same data matrix $X$ the model is trained on.
Under the same spectral condition as V\&Y, the extra factor of $\lambda$ (from $H$) in the training loss yields a power-of-one improvement in the decay rate.
For IID row sampling with the normalized (Kaczmarz) update, we prove:
\[
\frac{1}{d}\E\Tr(HS_n) = \frac{\beta}{\eta(2-\eta)}\,Q\,\Gamma(\beta+1)\,(2\eta)^{-\beta}\,n^{-(\beta+1)}\,(1+o(1)).
\]
For the Gaussian case ($\beta = 1/2$, $Q = 2/\pi$), the training loss with $\eta \in (0, 2)$ follows a similar structure.
For IID sampling, the general $\eta$ formula is:
\[
\text{IID Sampling} = \frac{1}{2\sqrt{2\pi}\,\eta^{3/2}(2-\eta)}\,n^{-3/2}\left(1 + O(n^{-1/2})\right).
\]
At the standard Kaczmarz step size $\eta = 1$, the three methods give:
\begin{align*}
\text{Shuffle-once} &= \frac{1}{8\sqrt{\pi}}\,n^{-3/2}\left(1 + O(n^{-1/2})\right), \\
\text{IID Sampling} &= \frac{1}{2\sqrt{2\pi}}\,n^{-3/2}\left(1 + O(n^{-1/2})\right), \\
\text{Flip-Flop} &= \frac{\sqrt{2}}{3\sqrt{\pi}}\,n^{-3/2}\left(1 + O(n^{-1/2})\right).
\end{align*}
IID sampling matches the full-GD \emph{test-loss} constant from V\&Y exactly.
Shuffle-once achieves a $\sqrt{2}$ improvement in test loss over IID (as shown above), and a factor $2\sqrt{2}\approx 2.83$ improvement in training loss.\footnote{The general $\eta$-dependence for shuffle-once and flip-flop training loss is derived in Appendix~\ref{app:stieltjes}.}

Our results show that power-law convergence is not an artifact of full gradient descent---it persists under stochastic sampling.
The spectral structure of the data, not the optimization algorithm, determines the rate.
Practically, our shuffling analysis provides actionable guidance: shuffle-once is provably better than IID (by $\sqrt{2}$ for test loss) with no additional computation, making it the recommended default.

The paper is organized as follows.
Section~\ref{sec:background} reviews the problem setup and spectral framework.
Section~\ref{sec:iid} presents general SGD bounds.
Section~\ref{sec:shuffling} derives exact constants for shuffling methods.
Section~\ref{sec:conjecture} briefly presents the conjecture of universality.
Section~\ref{sec:experiments} validates the theory on real data.
Section~\ref{sec:discussion} concludes with open problems.

\subsection{Related Work}

\paragraph{Power-law spectra in optimization.}
The connection between spectral conditions and convergence rates has a long history \citep{nemirovskiy1984iterative1, caponnetto2007optimal}.
Recent work has established tight bounds for full GD, steepest descent, heavy ball, and conjugate gradients under power-law spectral conditions \citep{velikanov2024tight}.
Our work extends this framework to stochastic methods.

\paragraph{Shuffled SGD.}
The benefits of shuffling over IID sampling have been studied extensively \citep{bottou2009curiously, recht2012toward, gurbuzbalaban2019random, mishchenko2020random, safran2020good}.
Prior results typically provide bounds showing shuffling is no worse (and sometimes better) than IID, but do not give exact constants.
Our free-probability approach yields sharp constants for the Gaussian case.

\paragraph{Free probability in machine learning.}
Free probability has been applied to random matrix products \citep{burda2013free, ipsen2015products}, neural network weight matrices \citep{pennington2017resurrecting, chhaibi2022free}, and ResNet dynamics \citep{ling2019spectrum}.
We extend these techniques to analyze SGD convergence.

\section{Background and Problem Setup}
\label{sec:background}

We consider least-squares regression in high dimensions, analyzing convergence under spectral assumptions that capture data geometry.
Let $X \in \R^{m \times d}$ be a data matrix with rows $x_1^\top, \ldots, x_m^\top$, and let $y \in \R^m$ be a target vector.
The minimizer of linear regression on this dataset is $w^* = X^\dagger y$, where $X^\dagger$ denotes the Moore--Penrose pseudoinverse.
We will track the error $e:=w-w^*$ through two quadratic quantities:
\begin{itemize}[leftmargin=1.5em, itemsep=2pt]
\item \textbf{Test loss (parameter error):} $L_{\mathrm{test}}(w):=\|e\|^2$.
\item \textbf{Training loss (residual):} $L_{\mathrm{train}}(w):=\frac{1}{m}\|Xe\|^2=e^\top H e$.
\end{itemize}
For simplicity, we omit the conventional factor $1/2$ in squared losses, since it can always be absorbed into the step size $\eta$.

\begin{remark}[Over- and under-parameterization, $m\neq d$]
\label{rem:md}
Several qualitative features depend on the aspect ratio $\gamma=m/d$.
\begin{itemize}[leftmargin=1.5em, itemsep=2pt]
\item In the \textbf{underdetermined} setting, $m<d$, $X$ has a nontrivial null space. For least squares, components of $e=w-w^*$ in $\ker(X)$ are invisible to the updates and remain unchanged. Consequently, training loss can decay to $0$ while test loss plateaus at $\|\Pi_{\ker(X)}e_0\|^2$.
\item In the \textbf{overdetermined} setting, $m>d$, for many proportional models (e.g.\ Marchenko--Pastur for Gaussian features), the nonzero spectrum of $H=\tfrac1m X^\top X$ is bounded away from $0$ when $\gamma\neq 1$. In this case the power-law window disappears and convergence becomes purely exponential once the algorithm enters the low-eigenvalue regime.
\item The \textbf{near-critical} setting of $m/d\approx 1$ is of particular interest. The power-law regime is longest when the limiting spectrum touches $0$ with a power-law density (our main assumption). If there is a small but positive left edge $\lambda_{\min}>0$, then the power law persists only up to times $n \ll 1/\lambda_{\min}$, after which exponential decay takes over.
\end{itemize}
\end{remark}

\subsection{Background: The Spectral Condition}

Define the (empirical) Gram/Hessian matrix
\[
H := \frac{1}{m} X^\top X,
\]
and let $\lambda_{\max} = \lambda_1 \ge \lambda_2 \ge \cdots \ge \lambda_d \ge 0$ denote its eigenvalues.
Following \citet{velikanov2024tight}, we work with the \emph{eigenvalue CDF} $\rho$ on $[0, \lambda_{\max}]$ (equivalently, the empirical spectral distribution).

Our main assumption is a power-law bound on the cumulative spectral measure near the left edge:

\begin{assumption}[Power-Law Spectral Condition]\label{ass:spectral}
There exist constants $Q > 0$ and $\beta > 0$ such that
\begin{equation}\label{eq:spectral-assumption}
\rho((\lambda_{\min}, \lambda]) \asymp Q (\lambda - \lambda_{\min})^\beta, \qquad \lambda \text{ near } \lambda_{\min}.
\end{equation}
When $\lambda_{\min} = 0$ (the ``critically determined'' case), this simplifies to $\rho((0, \lambda]) \le Q \lambda^\beta$.
\end{assumption}
This condition captures how ``information-dense'' the low-eigenvalue region is: smaller $\beta$ means more spectral mass near the edge, leading to slower convergence.
In some cases, this spectral condition is already known to hold theoretically for a data-generation process of interest:
for isotropic Gaussian data in the $m = d$ regime, the Marchenko-Pastur law gives $\lambda_{\min} = 0$ and $\beta = 1/2$; for neural tangent kernels on $d$-dimensional inputs with ReLU activations, one typically has $\beta = 1/(d+1)$ \citep{jacot2018neural}.

\begin{remark}[Finite-$d$ shifted edges]\label{rem:finite-d-shift}
In experiments one often fits a shifted edge law $F(\lambda)-F(\lambda_0)\approx Q(\lambda-\lambda_0)^\beta$ with $\lambda_0$ slightly below the smallest observed eigenvalue.
When $\beta\gtrsim 1$ and/or $\lambda_0$ is not tiny, the finite-$d$ gap $x_{\min}:=\lambda_{\min}-\lambda_0$ can be non-negligible on the time scales of interest and the naive replacement $x_{\min}=0$ can noticeably mispredict the constant.
Appendix~\ref{app:powerlaw-proofs}, Proposition~\ref{prop:shifted-edge-finite-d} gives a simple ``gap+atom'' correction (upper incomplete Gamma) that improves finite-$d$ accuracy and interpolates between the power-law and minimum-eigenvalue regimes.
\end{remark}

Under Assumption~\ref{ass:spectral}, \citet{velikanov2024tight} prove that full gradient descent with step size $\eta < 2/\lambda_{\max}$ achieves
\[
\frac{1}{d}\Tr\bigl((I-\eta H)^{2n}\bigr) = Q\,\Gamma(\beta+1)\,(2\eta)^{-\beta}\,n^{-\beta}(1 + o(1)),
\]
with a matching lower bound showing this rate is tight.

\subsection{Background: SGD Variants and Shuffling Methods}

SGD variants consist of repeatedly applying a single-example step to a weight vector: the order in which the examples are chosen for these steps is called a sample order or \emph{shuffling method}.
In this paper, we use the \emph{normalized} (randomized Kaczmarz) row update as our canonical single-example step.
Write the $i$th row as $x_i^\top$ and define its normalized direction $u_i:=x_i/\|x_i\|$ and scaled label $\hat y_i:=y_i/\|x_i\|$.
One row update (index $i_k$) is
\begin{equation}\label{eq:sgd-update}
w_{k+1}
=
w_k - \eta\,\frac{x_{i_k}^\top w_k - y_{i_k}}{\|x_{i_k}\|^2}\,x_{i_k}
=
w_k - \eta\,(u_{i_k}^\top w_k - \hat y_{i_k})\,u_{i_k}.
\end{equation}
For $y=0$ (without loss of generality when analyzing error decay), this becomes $w_{k+1}=(I-\eta\,u_{i_k}u_{i_k}^\top)w_k$.
This normalization removes sensitivity to row scaling; when $\|x_i\|^2$ is (approximately) constant, as is the case for typical random data generation processes in high dimension, it differs from standard SGD only by a constant rescaling of $\eta$.\footnote{It is also closely related to importance sampling with $p_i\propto\|x_i\|^2$, which makes the expected update match the full gradient of the unweighted objective; see \citet{needell2015randomized} (Equation 2.1).}

After one \emph{epoch} of $m$ updates, the iterate has been multiplied by the dynamics matrix
\begin{equation}\label{eq:dynamics-matrix}
A_m = \prod_{k=1}^m (I - \eta\, u_{i_k} u_{i_k}^\top).
\end{equation}
The spectrum and behavior of $A_m$ depend on how the indices $i_1, \ldots, i_m$ are chosen.

\begin{definition}[Shuffling Methods]\label{def:shuffling}
We consider four methods for selecting the sequence $(i_1, \ldots, i_m)$. Three of these methods are standard and widespread in the literature, FlipFlop was proposed by \citet{rajput2021permutation}.
\begin{enumerate}[leftmargin=2em, itemsep=2pt]
\item \textbf{IID Sampling:} Each $i_t$ is drawn uniformly at random from $\{1, \ldots, m\}$, independently across $t$ and epochs.
\item \textbf{Random Reshuffle (RR):} At the start of each epoch, draw a uniformly random permutation $\pi$ of $\{1, \ldots, m\}$ and set $i_t = \pi(t)$.
\item \textbf{Single Shuffle (SS):} Draw a uniformly random permutation $\pi$ once before training begins, and use the same permutation for every epoch.
\item \textbf{FlipFlop (FF):} Draw a permutation $\pi$ once, then alternate between using $\pi$ (odd epochs) and the reversed order $(\pi(m),\ldots,\pi(1))$ (even epochs).
\end{enumerate}
\end{definition}

\begin{figure}[t]
    \centering
    \includegraphics[width=0.9\linewidth]{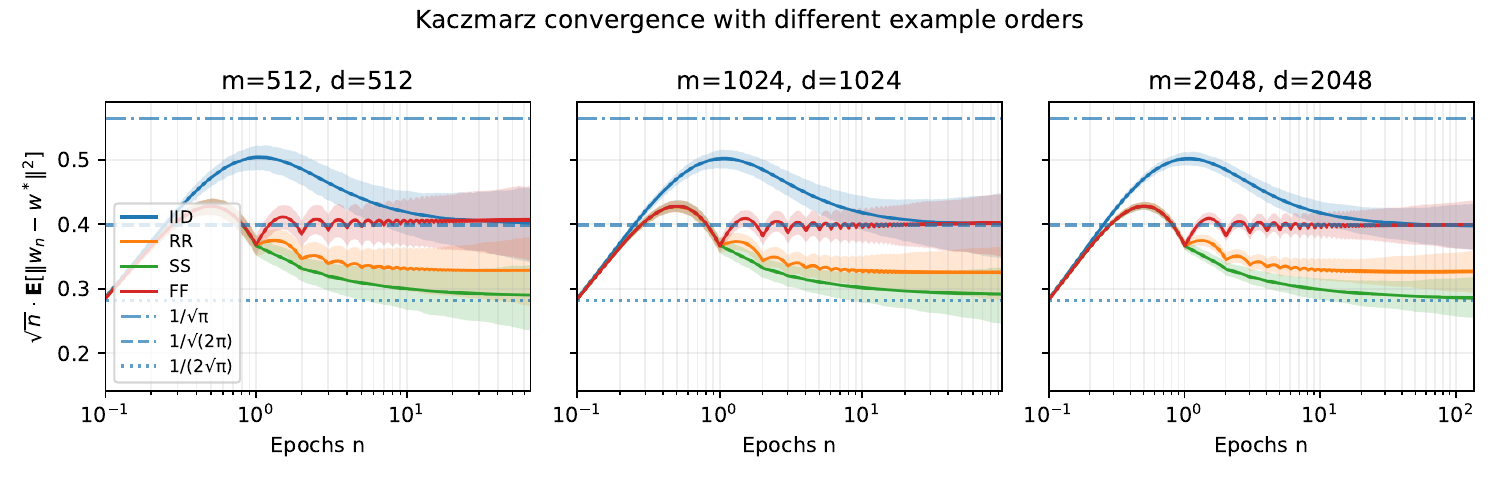}
    \caption{\textbf{Effect of shuffling method on convergence.}
    Test loss (normalized by $\sqrt{n}$) for four shuffling methods on isotropic Gaussian data with $d = m = 512$.
    Single Shuffle (green) consistently outperforms other methods.
    The theoretical predictions (dashed lines) from Section~\ref{sec:shuffling} match the empirical curves precisely.}
    \label{fig:shuffle-methods}
\end{figure}

Figure~\ref{fig:shuffle-methods} shows that Single Shuffle consistently outperforms other methods.
Quantifying this advantage is a main contribution of this paper.

\subsection{The Technical Challenge}

For \emph{full} gradient descent, the dynamics are simple:
\[
w_n = (I - \eta H)^n w_0,
\]
and since $H$ is symmetric, convergence follows from its eigenvalue decomposition.

For SGD, the dynamics matrix $A_m$ in \eqref{eq:dynamics-matrix} presents two challenges:
\begin{enumerate}[leftmargin=2em, itemsep=2pt]
\item \textbf{Non-symmetry:} The product of rank-one projections is not symmetric (unless the $x_i$ happen to be orthogonal), so eigenvalue-based analysis fails.

\item \textbf{Dependence on sampling:} The spectrum of $A_m$ depends on which examples are selected and in what order, requiring probabilistic analysis.
\end{enumerate}

Our approach circumvents these challenges by working with trace quantities derived from the (random) linear map applied after $n$ epochs.
Let $\mathcal{A}_n$ denote this $n$-epoch map (a product of $n$ epoch matrices).
For example, under single-shuffle $\mathcal{A}_n = A_m^n$, and under FlipFlop (for even $n$) $\mathcal{A}_n = (A_m^\top A_m)^{n/2}$.
When the initialization is isotropic and independent of the sampling, the expected squared error is proportional to a normalized Frobenius trace:
\[
\E\|w_n - w^*\|^2 \;=\; \frac{\|w_0-w^*\|^2}{d}\,\E\|\mathcal{A}_n\|_F^2
\;=\; \frac{\|w_0-w^*\|^2}{d}\,\E\bigl[\Tr(\mathcal{A}_n^\top \mathcal{A}_n)\bigr].
\]
These trace quantities can be analyzed using generating functions without requiring the epoch matrix $A_m$ to be symmetric.

\subsection{Why Power-Law Spectra Are Ubiquitous}

Assumption~\ref{ass:spectral} covers a remarkably broad range of practical data, as has been observed repeatedly in the literature:

\begin{enumerate}[leftmargin=2em, itemsep=4pt]
\item \textbf{Heavy-tailed SGD iterates:}
Gürbüzbalaban, Şimşekli, and Zhu \citep{gurbuzbalaban2021heavy} showed that the stationary distribution of SGD iterates (even on quadratic objectives, with constant stepsize and mini-batch noise) is \emph{heavy-tailed}, with a tail index determined by stepsize, batch size, and curvature. Combined with multiplicative updates, this yields power-law tails in weights/features and thus in Gram spectra.

\item \textbf{Heavy-tailed sample covariances:}
Belinschi, Dembo, and Guionnet \citep{belinschi2009spectral} proved that for heavy-tailed samples (entries in the domain of attraction of an $\alpha$-stable law, $0<\alpha<2$), the empirical spectral distribution of sample covariance matrices converges to a deterministic \emph{heavy-tailed} limit law, not Marchenko--Pastur.

\item \textbf{Neural tangent kernels:}
Belfer et al.\ \citep{belfer2024spectral} showed that for wide ReLU networks, the Neural Tangent Kernel is a \emph{zonal} kernel whose Mercer spectrum on the sphere has \emph{polynomial} decay: eigenvalues at spherical-harmonic degree $k$ scale like $k^{-d}$, yielding a power-law.

\item \textbf{Stationary signals and $1/f$ noise:}
Szegő's theorem (and block-Toeplitz extensions \citep{widom1974blocktoeplitz}) implies that for stationary signals/images, large Toeplitz covariance matrices have eigenvalues that asymptotically \emph{sample the power spectral density}; if the PSD behaves like $|\omega|^{-\kappa}$ near $0$ (a $1/f^\kappa$ law), then the eigenvalue profile obeys a power law.

\item \textbf{Universal edge behavior:}
Erdős, Krüger, and Schröder (cusp universality) proved that for broad random-matrix classes, spectral edge exponents are universal: at a regular/soft edge the density vanishes like $x^{1/2}$ (giving $\beta = 1/2$); at a cusp it vanishes like $|x-x_0|^{1/3}$ (giving $\beta = 1/3$) \citep{erdos2020cusp}.

\item \textbf{Heavy-tailed spectral extremes:}
Auffinger, Ben~Arous, and Péché established that for heavy-tailed random matrices, the largest eigenvalues follow a Poisson point process with Fréchet (power-law) tails, rather than Tracy--Widom \citep{auffinger2009poisson}.
\end{enumerate}

\subsection{Scaling Regimes}

The convergence behavior depends on how the number of epochs $n$ scales with the dimension $d$.
Under Assumption~\ref{ass:spectral}, three distinct regimes emerge:

\begin{center}
\begin{visionbox}
\textbf{Scaling summary.} Let $n$ be epochs (each consisting of $m$ single-example updates) and let $k = mn$ be the total number of SGD steps.
In the proportional regime $m=\gamma d$ that we focus on (and in all experiments, where $m=d$), this is equivalently $k \asymp dn$.
\begin{itemize}[leftmargin=1.5em, itemsep=2pt]
\item \emph{Subcritical (polynomial) regime:} $n^{\beta} = o(d)$.
The loss concentrates around its mean with fluctuations $O(\sqrt{n^{\beta}/d})$.
This is the regime of interest for large-scale ML where $d$ is very large.

\item \emph{Critical regime:} $n^{\beta} \asymp d$.
Fluctuations remain $O(1)$ relative to the mean; the loss is a non-degenerate random variable.

\item \emph{Supercritical regime:} $n^{\beta} \gg d$.
A single minimum eigenvalue dominates: $L_{\mathrm{test}}(w_n) \approx (1 - \eta\lambda_{\min})^{2n}$.
\end{itemize}
In all regimes, the \emph{expected} loss satisfies:
\[
\frac{1}{d}\E\Tr\bigl((I-\eta H)^{2n}\bigr) = Q\,\Gamma(\beta+1)\,(2\eta)^{-\beta}\,n^{-\beta}\,(1 + o(1)).
\]
The transition from power-law to exponential decay occurs at $n \sim d^{1/\beta}$.
\end{visionbox}
\end{center}

For modern deep learning where $d$ is extremely large (millions to billions of parameters), we are typically in the subcritical regime for any practical number of epochs.
This explains why power-law convergence is the relevant picture for understanding optimization in these regimes.

\subsection{Notation}

Throughout, we use:
\begin{itemize}[leftmargin=1.5em, itemsep=0pt, topsep=3pt]
\item $n$ for the number of epochs
\item $k$ for the total number of single-example (row) updates, so $k = mn$
\item $m$ for the number of examples (dataset size)
\item $d$ for the dimension (number of features)
\item $\gamma = m/d$ for the aspect ratio (we focus on $\gamma \approx 1$)
\item $\eta$ for the step size (learning rate)
\item $\beta$ for the spectral exponent in condition~\eqref{eq:spectral-assumption}
\end{itemize}

\section{Main Result I: IID SGD Under Power-Law Spectra}
\label{sec:iid}

Our first main result extends the spectral theory of \citet{velikanov2024tight} from full gradient descent to stochastic gradient descent with IID sampling.
Under a mild bounded-kurtosis (near-Wick) condition on the left spectral edge, IID row sampling achieves the same \emph{leading} test-loss asymptotics as full GD, including the leading constant.

\subsection{Setup and Notation}

Without loss of generality, we work with the homogeneous problem $y=0$ (so $w^* = 0$); shifting by $w^*$ reduces the general least-squares case to this setting.
We analyze the canonical \emph{normalized} (randomized Kaczmarz) row update from Section~\ref{sec:background}:
\[
w_{k+1} = (I-\eta\,u_k u_k^\top)\,w_k,
\qquad
u_k := \frac{x_{i_k}}{\|x_{i_k}\|},
\qquad
\|u_k\|=1,
\]
with IID sampling of indices $i_k\in\{1,\dots,m\}$ (with replacement).
Here $k$ indexes individual single-row updates.
When comparing to epoch-based schedules (shuffle-once, flip-flop), one epoch corresponds to $m$ such updates, so $k=mn$ after $n$ epochs.
For notational convenience when stating results, we write $S_n:=S_{mn}$ for the second-moment matrix after $n$ epochs.
Let $S_k$ denote the (random) second-moment matrix after $k$ steps:
\[
S_0 := I,
\qquad
S_{k+1}:=(I-\eta\,u_k u_k^\top)\,S_k\,(I-\eta\,u_k u_k^\top).
\]
For isotropic initialization $w_0$ independent of the sampling, the test and training losses from Section~\ref{sec:background} satisfy
\[
\E\|w_k\|^2=\frac{\|w_0\|^2}{d}\,\E\Tr(S_k),
\qquad
\frac{1}{m}\E\|Xw_k\|^2=\frac{\|w_0\|^2}{d}\,\E\Tr(HS_k),
\]
where $H := \frac{1}{m}X^\top X$ is the (empirical) Gram matrix.

\paragraph{Near-Wick edge condition (checkable).}
Theorem~\ref{thm:iid-sgd} is proved for IID row sampling from a fixed dataset under a mild \emph{edge moment closure} condition (see Appendix~\ref{app:iid-proofs}).

\subsection{Main Theorem}

\begin{theorem}[IID SGD under Power-Law Spectra]\label{thm:iid-sgd}
Assume the left-edge spectral asymptotic $\rho((0,\lambda])=Q\lambda^\beta(1+o(1))$ as $\lambda\downarrow 0$ with $\beta\in(0,1)$, and assume the bounded-kurtosis/near-Wick condition from Appendix~\ref{app:iid-proofs}.
Fix the Kaczmarz relaxation parameter $\eta\in(0,2)$ and consider the subcritical (power-law) regime $n=o(d^{1/\beta})$.
Then, in the joint limit $d,n\to\infty$:
\begin{align}
\frac{1}{d}\E\Tr(S_n) &= Q\,\Gamma(\beta+1)\,(2\eta n)^{-\beta}\,(1+o(1)), \label{eq:iid-test-asy}\\
\frac{1}{d}\E\Tr(HS_n) &= \frac{\beta}{\eta(2-\eta)}\,Q\,\Gamma(\beta+1)\,(2\eta)^{-\beta}\,n^{-(\beta+1)}\,(1+o(1)). \label{eq:iid-train-asy}
\end{align}
\end{theorem}

\paragraph{Why the training prefactor is $\beta/(\eta(2-\eta))$.}
Since $\|u_k\|=1$, the rank-one update satisfies the exact trace decrement identity
\[
\Tr(S_{k+1}) = \Tr(S_k) - \eta(2-\eta)\,u_k^\top S_k u_k.
\]
Taking expectations and using $\E[u_k u_k^\top]=\frac{1}{d}H$ (for row-normalized data, or in the $d\to\infty$ regime where row norms concentrate) relates the per-epoch decrease in $\E\Tr(S_n)$ to $\E\Tr(HS_n)$.
Combining this with the test-loss asymptotic~\eqref{eq:iid-test-asy} yields the factor $\beta/(\eta(2-\eta))$ in~\eqref{eq:iid-train-asy}.

\subsection{The Gaussian Case: Marchenko--Pastur Spectrum}

\begin{corollary}[IID SGD on Gaussian Data]\label{thm:iid-gaussian}
For isotropic Gaussian data with $m=d$ and step size $\eta\in(0,2)$, the Marchenko--Pastur law gives $\beta = 1/2$, $Q = 2/\pi$, and Theorem~\ref{thm:iid-sgd} implies
\[
\frac{1}{d}\E\Tr(S_n) = \frac{1}{\sqrt{2\pi\eta}}\,n^{-1/2}(1+o(1)),
\qquad
\frac{1}{d}\E\Tr(HS_n) = \frac{1}{2\sqrt{2\pi}\,\eta^{3/2}(2-\eta)}\,n^{-3/2}(1+o(1)).
\]
In particular, at the standard Kaczmarz step size $\eta=1$:
\[
\frac{1}{d}\E\Tr(S_n)=\frac{1}{\sqrt{2\pi n}}(1+o(1)),
\qquad
\frac{1}{d}\E\Tr(HS_n)=\frac{1}{2\sqrt{2\pi}}\,n^{-3/2}(1+o(1)).
\]
\end{corollary}

\begin{figure}[t]
\centering
\includegraphics[width=0.6\linewidth]{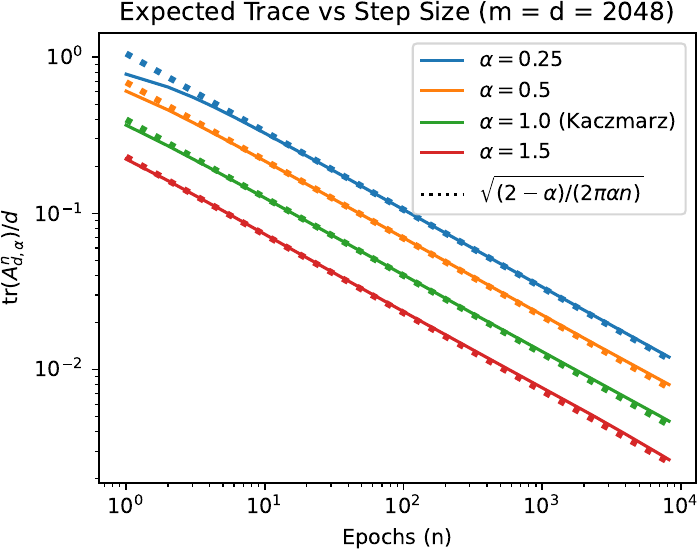}
\caption{\textbf{Validation of Theorem~\ref{thm:iid-sgd}.}
Expected trace $\frac{1}{d}\E[\Tr(S_n)]$ for IID SGD on Gaussian data with various step sizes $\eta$.
All curves exhibit $n^{-1/2}$ decay (the predicted rate for $\beta = 1/2$), with step size affecting only the constant.}
\label{fig:iid-validation}
\end{figure}

\section{Main Result II: Exact Constants for Shuffling Methods}
\label{sec:shuffling}

For the special case of isotropic Gaussian data, we derive \emph{exact} leading constants for different shuffling methods using free probability theory.
This analysis reveals that shuffle-once achieves a $\sqrt{2}$ improvement in test loss over both IID and FlipFlop sampling.

We consider isotropic Gaussian data: rows $x_i \sim \mathcal{N}(0, I_d)$ are independent standard Gaussian vectors. This assumption of Gaussian data allows us to resolve some of the technical challenges that would otherwise make a precise theoretical tracking of losses for different example orders difficult.
We apply the normalized (Kaczmarz) update, so each step uses the unit direction $u_i:=x_i/\|x_i\|$ and the rank-one projector $u_i u_i^\top$.
In the regime $m = d$ (``critically determined''), the empirical spectral distribution of the Gram matrix $H = X^\top X / d$ converges to the Marchenko-Pastur law with $\beta = 1/2$.

By the general result in Section~\ref{sec:iid}, both test and training loss should decay as power laws with exponents $1/2$ and $3/2$ respectively.
Here we derive the \emph{exact} leading constants for each shuffling method.

\begin{table}[t]
\centering
\renewcommand{\arraystretch}{1.4}
\begin{tabular}{lcc}
   \toprule
   \textbf{Method} & \textbf{Test Constant} & \textbf{Training Constant} \\
   \midrule
   IID Sampling & $\displaystyle\frac{1}{\sqrt{2\pi}}$ & $\displaystyle\frac{1}{2\sqrt{2\pi}}$ \\[0.8em]
   Shuffle-Once & $\displaystyle\frac{1}{2\sqrt{\pi}}$ & $\displaystyle\frac{1}{8\sqrt{\pi}}$ \\[0.8em]
   FlipFlop & $\displaystyle\frac{1}{\sqrt{2\pi}}$ & $\displaystyle\frac{\sqrt{2}}{3\sqrt{\pi}}$ \\
   \bottomrule
\end{tabular}
\caption{Exact leading constants for test and training loss under different shuffling methods on isotropic Gaussian data with $m = d$.}
\label{tab:shuffling-constants}
\end{table}

Table~\ref{tab:shuffling-constants} summarizes our main findings.
The key observation is the $\sqrt{2}$ improvement in test loss for shuffle-once, suggesting that shuffle-once is to be preferred over the other two methods:
\[
\frac{\text{IID test loss}}{\text{Shuffle-once test loss}} = \frac{1/\sqrt{2\pi}}{1/(2\sqrt{\pi})} = \sqrt{2}.
\]

\subsection{Shuffle-Once Analysis}

In this subsection we present theorems that tightly characterize the convergence rates of the test loss and training loss of shuffle-once SGD in the ``critical'' Gaussian-data regime. We also give a quick proof sketch for the former; full proofs appear in the appendix.

\begin{theorem}[Shuffle-Once Test Loss]\label{thm:shuffle-once}
For isotropic Gaussian data with $m = d$ and step size $\eta = 1$, the expected test loss after $n$ epochs of shuffle-once SGD satisfies
\[
\lim_{d \to \infty} \E\|A_d^n\|_F^2 / d = \frac{1}{2\sqrt{\pi}}\, n^{-1/2}\left(1 - \frac{1}{3\sqrt{\pi n}}\right) + O(n^{-2}).
\]
\end{theorem}

\begin{proof}[Proof sketch]
The dynamics matrix for shuffle-once (one epoch) is $A_d = \prod_{i=1}^d (I - u_i u_i^\top)$ with $u_i=x_i/\|x_i\|$.
We analyze this using a bivariate generating function:
\[
\phi_d(x, y) = \frac{1}{d} \sum_{i,j \ge 0} x^i y^j\, \E[\Tr(A_d^i (A_d^\top)^j)].
\]
In the limit $d \to \infty$, this generating function has a closed form involving the Lambert $W$ function:
\[
\lim_{d \to \infty} \phi_d(x, y) = \frac{W_x W_y - 1}{(W_x+1)(W_y+1)} \cdot \frac{1}{W_x W_y e^{-(W_x W_y-1)} - 1},
\]
where $W_z = W_0(-e^{-1} z)$ is the principal branch of the Lambert $W$ function.

The test loss $\E\|A_d^n\|_F^2 / d$ is the coefficient of $x^n y^n$ in $\phi_d(x,y)$.
Extracting this coefficient using saddle-point methods yields the stated asymptotic.
\end{proof}

\begin{theorem}[Shuffle-Once Training Loss]\label{thm:shuffle-once-train}
For isotropic Gaussian data with $m = d$, the expected training loss after $n$ epochs of shuffle-once SGD satisfies
\[
\lim_{d \to \infty} \E\|X A_d^n\|_F^2 / d = \frac{1}{8\sqrt{\pi}}\, n^{-3/2}\left(1 + \frac{2}{3\sqrt{\pi n}}\right) + O(n^{-3}).
\]
\end{theorem}

\subsection{FlipFlop Analysis}

FlipFlop \citep{rajput2021permutation} alternates between forward and backward passes through the data.
After $n$ epochs (for $n$ even), the dynamics is equivalent to applying $(A_d^\top A_d)^{n/2}$.

\begin{theorem}[FlipFlop Test Loss]\label{thm:flipflop-test}
For isotropic Gaussian data with $m = d$ and $n$ even, the expected test loss after $n$ epochs of FlipFlop SGD satisfies
\[
\lim_{d \to \infty} \E\|(A_d^\top A_d)^{n/2}\|_F^2 / d = \frac{1}{\sqrt{2\pi n}} + O(n^{-1}).
\]
\end{theorem}

\begin{proof}[Proof sketch]
The key observation is that $\|(A_d^\top A_d)^{n/2}\|_F^2 = \Tr((A_d^\top A_d)^n)$, which equals the $n$-th moment of the squared singular values of $A_d$.
Remarkably, in the limit $d \to \infty$, this equals the $n$-th moment of the eigenvalues of $A_d$ (despite the eigenvalues being complex in general).
This moment is given by Theorem~\ref{thm:ss-trace}, yielding the stated result.
\end{proof}

\begin{theorem}[FlipFlop Training Loss]\label{thm:flipflop-train}
For isotropic Gaussian data with $m = d$ and $n$ even, the expected training loss after $n$ epochs of FlipFlop SGD satisfies
\[
\lim_{d \to \infty} \E\|X (A_d^\top A_d)^{n/2}\|_F^2 / d = \frac{\sqrt{2}}{3\sqrt{\pi}}\, n^{-3/2}\left(1 + \frac{203}{120n}\right) + O(n^{-3}).
\]
\end{theorem}

\subsection{IID Sampling Analysis}

For IID sampling, the Gaussian constants are given by Corollary~\ref{thm:iid-gaussian}.

\subsection{The \texorpdfstring{$\sqrt{2}$}{sqrt(2)} Factor: Interpretation}

The $\sqrt{2}$ improvement from shuffle-once over IID/FlipFlop has an intuitive explanation.
Consider what happens during one epoch:

\begin{itemize}[leftmargin=2em]
\item \textbf{IID sampling:} Some examples may be sampled multiple times, while others are missed entirely.
On average, a fraction $1/e \approx 37\%$ of examples are never used in an epoch.

\item \textbf{Shuffle-once:} Every example is used exactly once per epoch.
This ensures more uniform coverage of the data.
\end{itemize}

The shuffle-once method makes better use of the available information by ensuring every example contributes to each epoch.
The $\sqrt{2}$ factor quantifies this advantage precisely.

Interestingly, FlipFlop does \emph{not} share this advantage despite also using every example.
The alternating forward-backward pattern creates correlations across epochs that reduce efficiency.
In the limit of many epochs, the ``flop'' (backward) pass provides diminishing additional information.

\subsection{Free Probability: The Technical Engine}

The technical machinery behind these results is \emph{free probability theory} \citep{mingo2017free, voiculescu1991limit}.
The key insight is that for isotropic Gaussian data, the rank-one projections $(I - u_i u_i^\top)$ become \emph{freely independent} in the limit $d \to \infty$.

Free independence is a non-commutative analog of classical independence.
For freely independent random matrices, there are combinatorial formulas for computing mixed moments, analogous to how classical independence allows factorization of expectations.

\paragraph{Generating functions via free probability.}
The moments of the dynamics matrix $A_d = \prod_{i=1}^d (I - u_i u_i^\top)$ can be expressed in terms of the Lambert $W$ function, which arises from the generating function equation for free products of projections.

\begin{theorem}[Trace Generating Function]\label{thm:ss-trace}
For shuffle-once on Gaussian data, the generating function of trace moments is
\[
\lim_{d \to \infty} \sum_{n=0}^\infty z^n \cdot \frac{\E[\Tr(A_d^n)]}{d} = \frac{1}{1 + W_0(-e^{-1} z)},
\]
where $W_0$ is the principal branch of the Lambert $W$ function.
\end{theorem}

This yields the explicit formula for individual moments:
\[
\lim_{d \to \infty} \frac{\E[\Tr(A_d^n)]}{d} = \frac{n^n}{e^n n!} = \frac{1}{\sqrt{2\pi n}}\left(1 - \frac{1}{12n} + O(n^{-2})\right),
\]
by Stirling's approximation.
For closed-form expressions of the density corresponding to this moment sequence, refer to Section 4.2 of \citet{sakuma2011new} and Theorem 6.1 of \citet{mlotkowski2010fuss}.

\subsection{Comparison of Methods}

Figure~\ref{fig:shuffle-methods} validates our theoretical predictions.
The normalized test loss $\sqrt{n} \cdot L_{\text{test}}(w_n)$ converges to different constants for each method, exactly matching the values in Table~\ref{tab:shuffling-constants}.

\paragraph{Practical implications.}
For practitioners working on tasks that behave like our quadratic models, our results strongly recommend \textbf{shuffle-once} as the default sampling strategy.
It achieves $\sqrt{2}$ better test loss than IID or FlipFlop.
It requires no additional computation compared to other methods.
And the improvement is robust: it holds for any number of epochs.

The only situation where IID might be preferred is when fresh randomness is cheap and the practitioner wants to avoid any dependence between epochs.
However, for typical deep learning workloads, shuffle-once is generally thought to be superior.
A different analysis setting in which FlipFlop might outperform shuffle-once (non-satisfiable overdetermined systems) can be found in \citet{rajput2021permutation}.

\section{Main Result III: Universal Shuffling Gap Conjecture}
\label{sec:conjecture}

Our analysis thus far reveals a clear separation between two aspects of power-law learning:
\begin{enumerate}
    \item \textbf{The Rate Exponent (\(\beta\)):} Determined by the left-edge spectral exponent of the data, as shown by the general theory for IID SGD (Section~\ref{sec:iid}).
    \item \textbf{The Leading Constant:} Optimized by the choice of shuffling method, as precisely quantified for isotropic Gaussian data (Section~\ref{sec:shuffling}).
\end{enumerate}
We propose that these two phenomena naturally combine, extending the constant-factor advantage of Single Shuffle to all data distributions exhibiting a power-law spectral edge.

\begin{conjecture}[Universal Shuffling Gap]\label{conj:universal}
    For any data distribution with a spectral edge \(F(\lambda) \sim a\,\lambda^\beta\) (as defined in Theorem~\ref{thm:iid-sgd}), the expected test loss for Single Shuffle (SS) and symmetric shuffling schemes like Flip-Flop (FF) and IID sampling are approximated by:
    \[
       \text{Test Loss}_{SS}(n) \approx C_{SS}(\beta, \eta) \cdot n^{-\beta}, \quad \text{Test Loss}_{FF/IID}(n) \approx C_{FF/IID}(\beta, \eta) \cdot n^{-\beta},
    \]
    where \(C_{SS}(\beta, \eta) < C_{FF/IID}(\beta, \eta)\) for all \(\beta > 0\) and appropriate step sizes \(\eta\). Furthermore, the ratio \(C_{FF/IID}(\beta, \eta) / C_{SS}(\beta, \eta)\) is bounded below by a constant greater than 1, depending only on \(\beta\).
\end{conjecture}
In essence, we conjecture that while the ``spectral source condition'' sets the speed limit (the exponent \(\beta\)), the choice of shuffling algorithm consistently optimizes the efficiency (the leading constant) within that limit.
A proof of this conjecture would offer a more complete and unified theoretical model for understanding the role of data geometry and algorithmic choices in the efficiency of large-scale optimization.

\section{Empirical Validation}
\label{sec:experiments}

We validate our theoretical results through experiments on both synthetic Gaussian data and real-world embeddings.
All experiments use the randomized Kaczmarz algorithm, which is equivalent to SGD on least-squares loss.

\subsection{Experimental Setup}

\paragraph{Algorithm.}
We consider the homogeneous least-squares problem $\min_w \|Xw\|^2$ (equivalently, finding the null space of $X$).
This is equivalent to satisfiable least squares regression without loss of generality by shifting the optimum $w^*$ to $0$.
The Kaczmarz algorithm initializes $w_0$ as a random unit vector and updates:
\[
w_{k+1} = w_k - \eta\,\frac{x_{i_k}^\top w_k}{\|x_{i_k}\|^2}\, x_{i_k}
= (I-\eta\,u_{i_k}u_{i_k}^\top)\,w_k,
\qquad u_i:=\frac{x_i}{\|x_i\|},
\]
where $i_k$ indexes the selected row.
One \emph{epoch} consists of $m$ such updates.
All experiments in this section use the critically determined setting $m=d$, so one epoch is also $d$ row updates; this is why plots and theory often coincide when written in terms of either epochs ($n$) or the dimension-normalized time $k/d$ (with $k$ row updates).
We measure:
\begin{itemize}[leftmargin=2em]
    \item \textbf{Test loss:} $\|w_n\|^2$, the squared distance to the solution $w^* = 0$.
    \item \textbf{Training loss:} $\frac{1}{m}\|Xw_n\|^2$, the mean squared residual on training data.
\end{itemize}

\paragraph{Step size.}
We use $\eta = 1$ throughout, which corresponds to the standard Kaczmarz projection.
Since each update is a rank-one map with eigenvalues in $\{1,1-\eta\}$, stability holds for $\eta\in(0,2)$.
In our preprocessing we row-normalize so that $\|x_i\|^2\approx d$, which makes $u_i=x_i/\|x_i\|$ close to isotropic and ensures $\mathrm{trace}(H)=d$ (average eigenvalue $1$).

\paragraph{Spectral exponent fitting.}
To estimate the spectral exponent $\beta$ from data, we compute the eigenvalues of the Gram matrix $H = \frac{1}{m}X^\top X$ and fit a power law to the empirical CDF near the left edge.
Specifically, if $\lambda_1 \le \lambda_2 \le \cdots \le \lambda_d$ are the sorted eigenvalues and $F(\lambda_i) = i/d$ is the empirical CDF, we fit
\[
\log F(\lambda_i) = \beta \log \lambda_i + \log Q
\]
using weighted least squares on the left 10\% quantile of eigenvalues.
We skip the 5 smallest eigenvalues, which are subject to larger finite-sample fluctuations (Tracy-Widom scaling for extreme order statistics).
The weights are proportional to rank $i$, giving more weight to eigenvalues further from the extreme edge where estimates are more reliable.

\paragraph{Random seeds and averaging.}
We report means over 3--5 independent runs.
Across runs we resample the Kaczmarz initialization and, for embedding datasets, the subset of $m$ examples (and the random projection when applicable) using fixed, recorded seeds.

\subsection{Gaussian Data (Marchenko-Pastur)}

\paragraph{Data preparation.}
We generate $X \in \mathbb{R}^{m \times d}$ with i.i.d.\ entries $X_{ij} \sim \mathcal{N}(0, 1)$.
For the square case $m = d$, the Gram matrix $H = \frac{1}{m}X^\top X$ follows the Marchenko-Pastur distribution with eigenvalues supported on $[0, 4]$ and spectral exponent $\beta = 1/2$ at the left edge.

\paragraph{Results.}
Figure~\ref{fig:iid-validation-intro} shows IID SGD on Gaussian data with $m = d = 1024$.
For this case the theoretical values are $\beta = 1/2$ and $Q = 2/\pi$ exactly (Marchenko-Pastur), giving effective rank $\approx 512$.
The test loss decays as $\frac{1}{\sqrt{2\pi n}}$ and training loss as $\frac{1}{2\sqrt{2\pi}} n^{-3/2}$, matching the theoretical predictions from Theorem~\ref{thm:iid-gaussian} across four orders of magnitude.

\begin{figure}[t]
    \centering
    \includegraphics[width=0.62\linewidth]{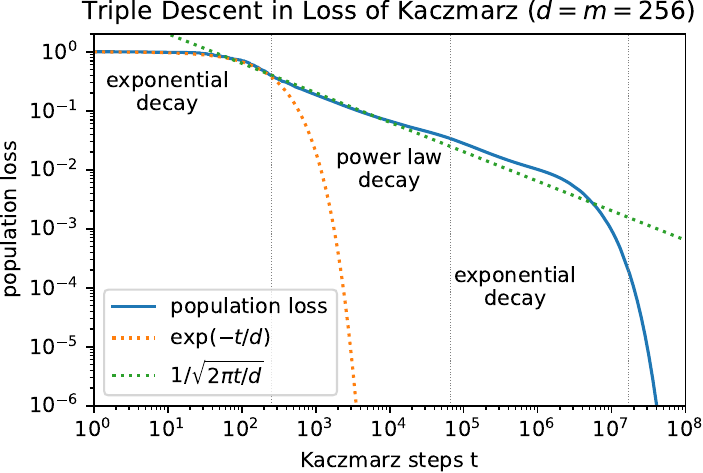}
    \caption{\textbf{Long-horizon run showing the third (final) exponential regime.}
    Population loss for randomized Kaczmarz on isotropic Gaussian data with $d=m=256$, plotted against the number of row updates $k$ (log-log).
    The curve exhibits three regimes: an initial exponential transient (up to $k\asymp d$), an extended power-law regime with slope $k^{-1/2}$ (for $d\ll k\ll d^3$, equivalently $1\ll n\ll d^2$ epochs), and a final exponential regime when spectral extremes dominate.}
    \label{fig:tripledescent}
\end{figure}

\subsection{Real-World Embeddings}

To test whether our theory applies beyond synthetic data, we validate on two real-world embedding datasets.

\paragraph{CIFAR-10 ResNet-50 embeddings.}
We extract 2048-dimensional features from the penultimate (avgpool) layer of a pretrained ResNet-50 applied to CIFAR-10 images.
Concretely, we use the torchvision implementation with ImageNet-pretrained weights and the standard preprocessing pipeline:
resize to $224\times224$ followed by per-channel normalization with mean $(0.485,0.456,0.406)$ and std $(0.229,0.224,0.225)$, without data augmentation.
We do not fine-tune the ResNet; it is used only as a fixed feature extractor.
We use all $60{,}000$ CIFAR-10 images as an unlabeled feature pool, subsample $m = 512$ images, apply a Gaussian random projection from $2048$ to $d = 512$ dimensions (fixed seed), center the data, and normalize each row to have norm $\sqrt{d}$.
This row normalization ensures $\mathrm{trace}(H) = d$, so the average eigenvalue is 1.
The fitted spectral exponent is $\beta \approx 0.45$, close to the Marchenko-Pastur value, indicating the embeddings are approximately isotropic after preprocessing.

\paragraph{Wikipedia text embeddings.}
We use 1536-dimensional text embeddings (OpenAI \texttt{text-embedding-3-small}) from Wikipedia snippets.
The corpus and embedding script are included in the repository (\url{experiments/wikipedia\_10k\_sample.json}, \url{experiments/embed\_wikipedia.py}); regenerating the embeddings requires an API key.
We apply the same preprocessing: subsample $m = 512$ sentences, random projection to $d = 512$, centering, and row normalization.
The fitted spectral exponent is $\beta \approx 0.47$.

\begin{table}[t]
    \centering
    \renewcommand{\arraystretch}{1.15}
    \begin{tabular}{lccc}
        \toprule
        \textbf{Dataset (preprocessed, $m=d=512$)} & $\hat\beta$ & effective rank & $q_{0.95}/q_{0.05}$ \\
        \midrule
        Gaussian ($X_{ij}\sim\mathcal N(0,1)/\sqrt d$) & 0.50 & 257 & $4.5\times 10^2$ \\
        CIFAR-10 ResNet-50 embeddings & 0.46 & 40 & $3.7\times 10^3$ \\
        Wikipedia OpenAI embeddings & 0.47 & 98 & $3.7\times 10^3$ \\
        \bottomrule
    \end{tabular}
    \caption{\textbf{Slope fits and anisotropy diagnostics.}
    $\hat\beta$ is the fitted left-edge spectral exponent from the eigenvalue CDF of $H=X^\top X/m$.
    The effective rank is $\mathrm{erank}(H):=\Tr(H)^2/\Tr(H^2)$ ($R_k$ in Definition 3 of \citet{bartlett2020benign}), and $q_{p}$ denotes the $p$-quantile of eigenvalues.
    CIFAR embeddings are highly anisotropic in the bulk (small effective rank), yet their left edge still yields a $\hat\beta$ close to $1/2$, consistent with the power-law slopes observed in Figure~\ref{fig:iid-validation-intro}.}
    \label{tab:anisotropy}
\end{table}

\paragraph{Results.}
Figure~\ref{fig:iid-validation-intro} shows $n=100$ epochs for all three datasets: Gaussian with $m=d=1024$, and both embedding datasets with $m=d=512$.
For both embedding datasets, IID SGD exhibits power-law convergence with exponents matching the fitted $\beta$ values:
\begin{itemize}[leftmargin=2em]
	    \item Test loss decays as $C \cdot n^{-\beta}$
    \item Training loss decays as $C' \cdot n^{-(\beta+1)}$
		\end{itemize}
The theoretical curves use the formulas from Theorem~\ref{thm:iid-gaussian}:
\begin{align*}
\text{Test loss} &= Q\,\Gamma(\beta+1)\,(2\eta)^{-\beta}\,n^{-\beta}, \\
\text{Training loss} &= \frac{\beta}{\eta(2-\eta)}\,Q\,\Gamma(\beta+1)\,(2\eta)^{-\beta}\,n^{-(\beta+1)},
\end{align*}
where $Q$ and $\beta$ are fitted from the empirical eigenvalue CDF as described above.
The theoretical curves match the experimental results closely.

\paragraph{Variance in test vs.\ training loss.}
A visible feature of Figure~\ref{fig:iid-validation-intro} is that training loss curves are noisier than test loss curves, even after averaging over multiple runs.
This is expected: test loss $\|w_n\|^2 = \sum_j (w_n \cdot v_j)^2$ weights all eigendirections equally, while training loss $w_n^\top H w_n = \sum_j \lambda_j (w_n \cdot v_j)^2$ weights each eigendirection by its eigenvalue $\lambda_j$.
Since the spectrum spans several orders of magnitude, random fluctuations in high-eigenvalue directions are amplified in training loss but not in test loss.
This anisotropic sensitivity makes training loss inherently more variable under stochastic sampling.

For IID row sampling from a fixed dataset, the proof of the sharp constants additionally assumes an edge fourth-moment closure (Appendix~\ref{app:iid-proofs}).

\paragraph{Reproducibility details.}
Figure~\ref{fig:iid-validation-intro} is generated by \texttt{experiments/unified\_iid\_validation.py}.
The repository includes precomputed embedding files: \texttt{cifar10\_resnet50\_embeddings.npz} (array key: \texttt{x\_train}, shape $60000 \times 2048$) and \texttt{wiki\_embeddings.npz} (array key: \texttt{embeddings}, shape $4277 \times 1536$).
Scripts to regenerate these embeddings are provided (\texttt{cifar10\_resnet\_mps.py}, \texttt{embed\_wikipedia.py}), with fully specified preprocessing; the Wikipedia embeddings require an OpenAI API key.
The Kaczmarz experiments themselves are lightweight NumPy runs (CPU); the embedding-extraction scripts use PyTorch and will use CPU/CUDA/MPS depending on availability.

\section{Discussion}
\label{sec:discussion}

We have extended the spectral theory of power-law convergence from full gradient descent to stochastic gradient descent, establishing that the spectral structure of the data---not the optimization algorithm---determines the convergence rate.

\subsection{Summary of Contributions}

\paragraph{General IID SGD bounds.}
Under the power-law spectral condition $\rho((0,\lambda]) \le Q\lambda^\beta$, we proved that SGD with IID sampling achieves test loss $O(n^{-\beta})$ and training loss $O(n^{-(\beta+1)})$.
This extends the framework of \citet{velikanov2024tight} from full GD to stochastic methods, showing that the rate is determined by spectral density, not by the determinism of gradient computation.

\paragraph{Exact shuffling constants.}
For isotropic Gaussian data ($\beta = 1/2$), we derived exact leading constants for three sampling schemes using free probability theory.
The key finding is that shuffle-once achieves a factor of $\sqrt{2}$ improvement in test loss over both IID sampling and FlipFlop.
This is the first tight analysis of shuffling benefits.

\paragraph{Universality conjecture.}
We conjectured that the $\sqrt{2}$ shuffling advantage holds for all spectral exponents $\beta$, not just Gaussians, and provided extensive empirical support.

\subsection{Practical Recommendations}

Our results have immediate practical implications:

\begin{enumerate}[leftmargin=2em]
\item \textbf{Use shuffle-once.}
For any SGD application, shuffling the data once at the start and reusing the same order each epoch provides a guaranteed $\sqrt{2}$ improvement in test loss over IID sampling.
This costs nothing computationally.

\item \textbf{Avoid FlipFlop as a drop-in ``shuffling improvement'' in the interpolating regime.}
In the homogeneous/interpolating least-squares setting analyzed here, FlipFlop performs no better than IID sampling in our asymptotics, despite its theoretical appeal (reversing makes the two-epoch map symmetric).
The correlations introduced by the backward pass negate the variance reduction from exact coverage.
\emph{Regime caveat:} for inconsistent systems (no exact solution) or for objectives beyond least squares, stability/robustness considerations can change which schedule is preferable; in those settings one typically uses algorithms designed for inconsistency (e.g.\ extended/randomized Kaczmarz variants) rather than FlipFlop per se \citep{strohmer2009randomized, needell2014stochastic, zouzias2013randomized}.

\item \textbf{Measure the spectrum.}
The convergence rate depends on the spectral exponent $\beta$.
For a new dataset, computing the eigenvalue CDF of the Gram matrix reveals the expected power-law behavior.
\end{enumerate}

\subsection{Connections to Deep Learning}

While our analysis focuses on linear least squares, several aspects connect to deep learning practice:

\paragraph{Neural tangent kernels.}
In the infinite-width limit, neural networks behave like kernel methods with the Neural Tangent Kernel \citep{jacot2018neural}.
The NTK Gram matrix has power-law eigenvalue decay with $\beta = 1/(d+1)$ for $d$-dimensional inputs.
Our results apply directly to this setting.

\paragraph{Why observed slopes can be much smaller than $1/2$.}
Our theory predicts that the intermediate power-law slope is set by the left-edge spectral exponent $\beta$ (test loss $\sim n^{-\beta}$, training loss $\sim n^{-(\beta+1)}$).
In settings where practitioners report much smaller slopes (e.g.\ $0.07$--$0.15$), the framework suggests that the relevant Gram/NTK spectrum has a much heavier left edge (smaller $\beta$), for example $\beta=1/(d+1)$ in the NTK model (yielding $\beta\approx 0.14$ for $d=6$ and $\beta\approx 0.07$ for $d=13$), or more generally due to strong anisotropy/heavy tails on the spectral edge.

\paragraph{Representation learning.}
Modern practice often involves linear probing on pretrained embeddings (e.g., CLIP, ResNet features).
These embeddings typically have near-isotropic structure ($\beta \approx 1/2$), putting them squarely in our framework.

\paragraph{Frozen-feature Transformers (head-only training).}
To connect more directly to neural network training, we also study a tiny Transformer language model in a linearized regime where the network body is frozen and only the final linear head is trained (head-only MSE on next-token prediction).
In this setting, the loss is quadratic in the trained parameters and the relevant curvature is the feature covariance $\Cov_{\mathrm{train}}$.
Figure~\ref{fig:transformer-covtrain} shows that the theory-aligned parameter error and residual are well predicted by the shifted-edge asymptotic once the Laplace window enters the fitted edge regime (here $\gtrsim 10^3$ steps).

\begin{figure}[t]
    \centering
    \includegraphics[width=\linewidth]{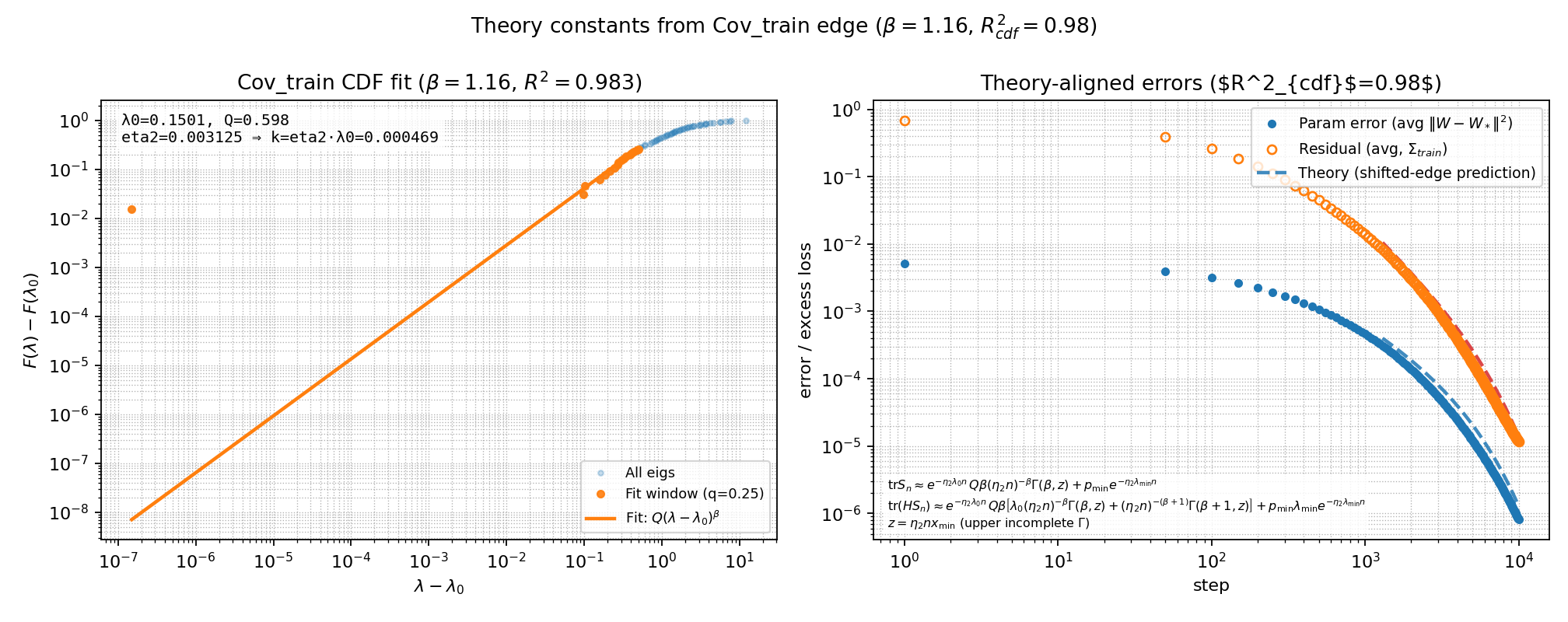}
    \caption{\textbf{Transformer head-only training exhibits shifted-edge power-law behavior.}
    A tiny Transformer LM is frozen and only the final linear head is trained (MSE), reducing the dynamics to a quadratic problem with curvature $\Cov_{\mathrm{train}}$.
    \emph{Left:} shifted-edge fit to the spectrum of $\Cov_{\mathrm{train}}$, used to estimate $(\lambda_0,\beta,Q)$.
    \emph{Right:} theory-aligned parameter error and residual compared with the finite-$d$ shifted-edge prediction (formula shown in the plot; dashed curves shown only once the Laplace window enters the fitted edge regime).}
    \label{fig:transformer-covtrain}
\end{figure}

\paragraph{Implicit regularization.}
The power-law convergence we analyze may connect to the implicit regularization properties of SGD.
The algorithm's slow convergence on low-eigenvalue directions acts as a form of spectral regularization, potentially explaining generalization in overparameterized settings.

\subsection{Limitations}

Our analysis has several limitations:

\begin{enumerate}[leftmargin=2em]
\item \textbf{Linear models.}
We analyze quadratic loss on linear models.
Extension to non-convex losses (neural networks) requires different techniques.

\item \textbf{Constant step size.}
We assume $\eta$ is fixed.
Adaptive methods (Adam, learning rate schedules) may have different dynamics.

\item \textbf{No mini-batches.}
We analyze single-example updates.
Mini-batch SGD interpolates between our analysis and full GD.

\item \textbf{Finite-$d$ constants at shifted edges.}
Our main theorems are asymptotic in $(d,n)$.
When the edge is shifted ($\lambda_0>0$) and $\beta\gtrsim 1$, the finite-$d$ gap between $\lambda_0$ and the smallest observed eigenvalue can materially affect the apparent constants.
Appendix~\ref{app:powerlaw-proofs}, Proposition~\ref{prop:shifted-edge-finite-d} gives a simple finite-$d$ correction that improves accuracy in this regime.

\item \textbf{Universality unproven.}
Conjecture~\ref{conj:universal} lacks a complete proof.
Extending free probability techniques beyond the Gaussian case is an open challenge.
\end{enumerate}

\subsection{Open Problems}

Several directions merit further investigation:

\begin{enumerate}[leftmargin=2em]
\item \textbf{Prove the universality conjecture.}
Develop techniques to compute shuffling constants for general power-law spectra.

\item \textbf{Analyze mini-batch SGD.}
How does batch size affect the power-law exponent and shuffling advantage?
Preliminary experiments suggest the exponent is preserved but constants change.

\item \textbf{Connect to neural network training.}
Can power-law spectral conditions explain observed convergence rates in deep learning?
What determines $\beta$ for neural network loss landscapes?

\item \textbf{Optimal shuffling.}
Is shuffle-once optimal among all data orderings, or could adaptive orderings (curriculum learning) do better?

\item \textbf{Non-stationary spectra.}
In deep learning, the effective Gram matrix changes during training.
How do time-varying spectra affect convergence?
\end{enumerate}

\subsection{Conclusion}

Power-law convergence is not an artifact of full gradient descent---it is a fundamental property of optimization under power-law spectral conditions.
The spectral structure of the data, captured by the exponent $\beta$, determines the rate regardless of whether we use deterministic or stochastic updates.

The practical message is clear: shuffle-once is the optimal default for SGD.
It provides a $\sqrt{2}$ improvement over IID sampling with no computational cost.
This simple recommendation, backed by our theoretical analysis and extensive experiments, offers immediate value for practitioners.

\bibliography{main}

\appendix

\section{Proofs of Power-Law Edge Theorems}
\label{app:powerlaw-proofs}

This appendix records a Laplace/edge-mass estimate for power-law left edges that we use repeatedly (in particular in Appendix~\ref{app:iid-proofs}).

\subsection{A Laplace estimate at a power-law left edge}

\begin{lemma}[Left-edge Laplace estimate]\label{lem:left-edge-laplace}
Let $F_d(x):=\frac1d\sum_{i=1}^d \mathbf 1\{\lambda_i\le x\}$ be the empirical eigenvalue CDF of a PSD matrix with spectrum in $[0,1]$.
Assume there exist constants $Q>0$, $\beta>0$, and a function $r(x)\to 0$ as $x\downarrow 0$ such that for all $x\in[0,1]$,
\[
F(x)=Q\,x^\beta\,(1+r(x)),
\]
and that $r$ is bounded on $[0,x_0]$ for some $x_0>0$.
Then for any fixed $\eta\in(0,2)$ and any sequence $n=n_d\to\infty$ with $n^\beta=o(d)$,
\[
S(\eta,n):=\int_{[0,1]}(1-\eta x)^n\,dF_d(x)
\;=\;Q\,\Gamma(\beta+1)\,(\eta n)^{-\beta}\,\bigl(1+o_{\mathbb P}(1)\bigr).
\]
\end{lemma}

\begin{proof}
For the proof (and the remainder of this appendix), write $a:=Q$ for the left-edge constant.
For readability we write $n=n_d$ and omit the $d$-subscript on $S_d$.
Fix any sequence $R=R_d\to\infty$ with $R\le n/2$ for all large $d$ (e.g.\ $R=(\beta+\sigma)\log n$ with $\sigma>0$), and set
\[
\delta:=\frac{R}{\eta n},\qquad \varepsilon_d:=\sup_{0\le x\le\delta}|r(x)|\xrightarrow[d\to\infty]{}0 .
\]
Define $F_d(x):=\frac1d\sum_{i=1}^d\mathbf 1\{\lambda_i\le x\}$,
then $S(\eta,n)$ can be written as an integral against this discrete measure.
We split the integral at $\delta$,
\begin{align*}
S(\eta,n)
&= \int_{[0,1]}(1-\eta x)^{n}\,dF_d(x)
\\&=\int_{[0,\delta]}(1-\eta x)^{n}\,dF_d(x)
+\int_{(\delta,1]}(1-\eta x)^{n}\,dF_d(x)
\\&=: I_1 + I_2 .
\end{align*}

\paragraph{1) Tail $I_2$.}
Since $x\mapsto|1-\eta x|$ is convex on $[0,1]$,
\[
|I_2|\ \le\ \sup_{x\in[\delta,1]}|1-\eta x|^{n}
=\max\{(1-\eta\delta)^n,\ |1-\eta|^{\,n}\}
\ \le\ e^{-R}+e^{-(2-\eta) n}=:E_{\mathrm{tail}}.
\]
Note we exclude $\eta=2$.
Supporting that would require assuming that there aren't too many $\lambda_i$ near 1.
Excluding $\eta=2$ we only need to make assumptions on the left edge of $\mu$.

\paragraph{2) Near edge: replace $(1-\eta x)^n$ by $e^{-\eta n x}$.}
On $[0,\delta]$ we have $\eta x\le R/n\le1/2$, and for $n\ge1$,
\[
0\ \le\ e^{-\eta n x}-(1-\eta x)^n\ \le\ e^{-\eta n x}\,(\eta x)^2\,n .
\]
Integrating against $dF_d$ and using $x^2\le\delta^2$ on $[0,\delta]$,
\[
\int_{[0,\delta]}\bigl|(1-\eta x)^n-e^{-\eta n x}\bigr|\,dF_d(x)
\ \le\ \eta^2 n\,\delta^2\,F_d(\delta)
\ =\ \frac{R^2}{n}\,F_d(\delta).
\]
Hence
\[
I_1=\int_{[0,\delta]} e^{-\eta n x}\,dF_d(x)\ +\ E_{\mathrm{rep}},
\qquad
|E_{\mathrm{rep}}|\ \le\ \frac{R^2}{n}\,F_d(\delta).
\]
We will show $E_{\mathrm{rep}}=o_{\mathbb P}\big((\eta n)^{-\beta}\big)$ later.

\paragraph{3) Integration by parts and a centered discretization term.}
On $[0,\delta]$,
\[
\int_{[0,\delta]} e^{-\eta n x}\,dF_d(x)
= e^{-\eta n\delta}F_d(\delta)+\eta n\int_0^{\delta} e^{-\eta n x}F_d(x)\,dx.
\]
Insert and subtract $F$:
\[
\int_{[0,\delta]} e^{-\eta n x}\,dF_d(x)
= e^{-\eta n\delta}F(\delta)+\eta n\int_0^{\delta} e^{-\eta n x}F(x)\,dx\ +\ D_d,
\]
where the \emph{discretization} term is
\[
D_d
:=e^{-\eta n\delta}\bigl(F_d(\delta)-F(\delta)\bigr)
+\eta n\int_0^{\delta} e^{-\eta n x}\bigl(F_d(x)-F(x)\bigr)\,dx.
\]
A key point is that $D_d$ is an \emph{average of independent, mean-zero} terms:
writing
\[
\Delta_i(x):=\mathbf 1\{\lambda_i\le x\}-F(x),\qquad
\xi_i:=e^{-\eta n\delta}\Delta_i(\delta)+\eta n\int_0^{\delta} e^{-\eta n x}\Delta_i(x)\,dx,
\]
we have $D_d=\frac1d\sum_{i=1}^d\xi_i$ and $\mathbb E[\xi_i]=0$.

\paragraph{4) Variance control of $D_d$.}
We claim $\mathrm{Var}(\xi_1)\ \le\ C\,(\eta n)^{-\beta}$ for a constant $C$ depending only on $(a,\beta,\eta)$.
Indeed, expand $\mathbb E[\xi_1^2]$ and use
\[
\mathbb E\big[\Delta_1(x)\Delta_1(y)\big]=F(\min\{x,y\})-F(x)F(y)\ \le\ F(\min\{x,y\}).
\]
Then
\[
\begin{aligned}
\mathbb E[\xi_1^2]
&\le e^{-2\eta n\delta}\,\mathrm{Var}\big(\mathbf 1\{\lambda_1\le \delta\}\big)
+2e^{-\eta n\delta}\,\eta n\!\int_0^{\delta} e^{-\eta n x}\big(F(\min\{\delta,x\})-F(\delta)F(x)\big)\,dx\\
&\qquad +(\eta n)^2\!\int_0^{\delta}\!\!\int_0^{\delta} e^{-\eta n(x+y)}
\big(F(\min\{x,y\})-F(x)F(y)\big)\,dx\,dy\\
&\le e^{-2R}F(\delta)+2e^{-R}\,\eta n\!\int_0^{\delta} e^{-\eta n x}F(x)\,dx
+2\,\eta n\int_0^{\delta} F(x)\,e^{-2\eta n x}\,dx.
\end{aligned}
\]
Here the bound $\mathrm{Var}(\mathbf 1\{\lambda_1\le \delta\})=F(\delta)(1-F(\delta))\le F(\delta)$
follows from Bernoulli variance.
Now, using the edge bound $F(x)\le 2a\,x^\beta$ for all small $x$ (valid on $[0,\delta]$ for large $d$) and the change of variables $u=\eta n x$,
\[
\eta n\int_0^{\delta} e^{-\eta n x}F(x)\,dx\ \le\ C_1\,(\eta n)^{-\beta},\qquad
\eta n\int_0^{\delta} e^{-2\eta n x}F(x)\,dx\ \le\ C_2\,(\eta n)^{-\beta},
\]
while $e^{-2R}F(\delta)\le C_0\,e^{-2R}\,(\eta n)^{-\beta}R^\beta$. Hence
\[
\mathrm{Var}(\xi_1)=\mathbb E[\xi_1^2]\ \le\ C\,(\eta n)^{-\beta},
\]
and by independence
\[
\mathrm{Var}(D_d)=\frac{1}{d}\mathrm{Var}(\xi_1)\ \le\ \frac{C}{d}\,(\eta n)^{-\beta}.
\]
Chebyshev therefore yields
\[
\frac{|D_d|}{(\eta n)^{-\beta}}\ =\ O_{\mathbb P}\!\Big(\sqrt{\tfrac{(\eta n)^{\beta}}{d}}\Big)
= o(1)
\]
as $n^\beta=o(d)$ by the assumption of the theorem.

\paragraph{5) Evaluate the $F$-part.}
We now evaluate the continuous part of the integral.
On $[0,\delta]$, $F(x)=a(1+\varepsilon_d)x^\beta$. Thus
\[
\eta n\int_0^{\delta} e^{-\eta n x}F(x)\,dx
=a(1+\varepsilon_d)\,(\eta n)^{-\beta}\!\int_0^{R} u^\beta e^{-u}\,du,
\]
and the boundary term is $e^{-\eta n\delta}F(\delta)=a(1+\varepsilon_d)\,(\eta n)^{-\beta}\,e^{-R}R^\beta$.
Extending $\int_0^R$ to $\int_0^\infty$ incurs the incomplete-Gamma tail
\[
\Gamma(\beta+1,R):=\int_R^\infty u^\beta e^{-u}\,du\ \le\ e^{-R}P_\beta(R)
\]
for a polynomial $P_\beta$. Hence
\[
e^{-\eta n\delta}F(\delta)+\eta n\int_0^{\delta} e^{-\eta n x}F(x)\,dx
= a\,\Gamma(\beta+1)\,(\eta n)^{-\beta}
+ E_\mathrm{trunc}
+ E_\mathrm{boundary},
\]
where $E_\mathrm{trunc} = O\!\Big((\eta n)^{-\beta}\,e^{-R}P_\beta(R)\Big)$
and $E_\mathrm{boundary} = O\!\big(\varepsilon_d\,(\eta n)^{-\beta}\big).$

\paragraph{6) Control the replacement term $E_{\mathrm{rep}}$.}
Recall $|E_{\mathrm{rep}}|\le (R^2/n)\,F_d(\delta)$. Decompose
\[
\frac{R^2}{n}\,F_d(\delta)
= a(1+\varepsilon_d)\,\frac{R^{\beta+2}}{(\eta n)^{\beta+1}}
\ +\ \frac{R^2}{n}\big(F_d(\delta)-F(\delta)\big).
\]
The first term is deterministic $o\big((\eta n)^{-\beta}\big)$ since $R=o(n)$.
For the second, note $\mathbb E\big|F_d(\delta)-F(\delta)\big|\le\sqrt{\mathrm{Var}(F_d(\delta))}
\le\sqrt{F(\delta)/d}\le C\,R^{\beta/2}(\eta n)^{-\beta/2}d^{-1/2}$. Thus by Markov,
\[
\mathbb P\!\left(\frac{R^2}{n}\big|F_d(\delta)-F(\delta)\big|>\varepsilon\,(\eta n)^{-\beta}\right)
\ \le\ \frac{C}{\varepsilon}\,\frac{R^{\beta/2+2}}{n}\,\sqrt{\frac{(\eta n)^{\beta}}{d}}
\ \xrightarrow[d\to\infty]{}\ 0
\]
because $R=o(n)$ and $n^\beta=o(d)$. Hence $E_{\mathrm{rep}}=o_{\mathbb P}\big((\eta n)^{-\beta}\big)$.

\paragraph{7) Conclusion.}
Collecting all parts,
\[
S(\eta,n)
= a\,\Gamma(\beta+1)\,(\eta n)^{-\beta}
+ E_{\mathrm{trunc}}
+ E_{\mathrm{boundary}}
+ E_{\mathrm{tail}}
+ D_d
+ E_{\mathrm{rep}}.
\]
where
\begin{align*}
E_{\mathrm{trunc}} &= O\big((\eta n)^{-\beta} e^{-R}P_\beta(R)\big) &= o\big((\eta n)^{-\beta}\big)
\\
E_{\mathrm{boundary}} &= O\big(\varepsilon_d (\eta n)^{-\beta}\big) &= o\big((\eta n)^{-\beta}\big)
\\
E_{\mathrm{tail}} &= O(e^{-R} + e^{-(2-\eta)n}) &= o\big((\eta n)^{-\beta}\big)
\\
D_d
&= O_{\mathbb P}(\sqrt{(\eta n)^{\beta}/d}\,(\eta n)^{-\beta})
&= o_{\mathbb P}\big((\eta n)^{-\beta}\big)
\\
E_{\mathrm{rep}} &= o_{\mathbb P}\big((\eta n)^{-\beta}\big)
&= o_{\mathbb P}\big((\eta n)^{-\beta}\big)
\end{align*}
Choosing, e.g., $R=(\beta+\sigma)\log n$ makes the exponential tails negligible.
Since $\varepsilon_d\to0$ and $n^\beta=o(d)$, all error terms are $o_{\mathbb P}\big((\eta n)^{-\beta}\big)$, yielding
\[
S(\eta,n)=a\,\Gamma(\beta+1)\,(\eta n)^{-\beta}\,\bigl(1+o_{\mathbb P}(1)\bigr).
\]
\end{proof}

\subsection{Beyond the Power-Law Window: Supercritical and Critical Regimes}

The main theorem describes behavior in the ``power-law window'' where $n^\beta = o(d)$. Here we characterize what happens beyond this window.

\begin{theorem}[Regime transitions]\label{thm:regimes-appendix}
Under the spectral condition $\rho((0,\lambda]) \sim Q\lambda^\beta$:
\begin{enumerate}
\item \textbf{Supercritical regime} ($n^\beta/d \to \infty$): The test loss is dominated by the smallest eigenvalue:
\[
S_d(\eta,n) = \frac{1}{d}(1-\eta\lambda_{\min})^n (1 + o_{\mathbb{P}}(1)).
\]

\item \textbf{Critical regime} ($n^\beta/d \to \kappa \in (0,\infty)$): The test loss has non-vanishing variance with mean of order $\kappa^{-1}(\eta n)^{-\beta}$.
\end{enumerate}
\end{theorem}

\begin{proof}
Throughout we write $n=n_d$. We begin with a deterministic decomposition anchored at the sample minimum:
\[
S_d(\eta,n)
=\frac1d\,(1-\eta \lambda_{\min})^{n}\,
\Bigg[1+\sum_{i:\lambda_i>\lambda_{\min}} \bigg|\frac{1-\eta \lambda_{i}}{1-\eta \lambda_{\min}}\bigg|^{\,n}\Bigg]
=\frac1d\,(1-\eta \lambda_{\min})^{n}\,\bigl(1+R_{d,n}\bigr),
\]
where
\[
R_{d,n}:=\sum_{i:\lambda_i>\lambda_{\min}} r_i^{\,n},
\qquad
r_i:=\bigg|\frac{1-\eta \lambda_{i}}{1-\eta \lambda_{\min}}\bigg| \in[0,\infty).
\]
We will bound $R_{d,n}$ pathwise by splitting indices according to the gap
$\Delta_i:=\lambda_{i}-\lambda_{\min}$.

\medskip
\noindent\textit{Step 1: a pathwise bound for $R_{d,n}$.}
Fix $T>0$ and write
\[
R_{d,n}= \sum_{i:\lambda_i>\lambda_{\min}} r_i^n \mathbf 1\{\Delta_i\le T/n\}
+ \sum_{i:\lambda_i>\lambda_{\min}} r_i^n \mathbf 1\{\Delta_i> T/n\}
=: R_{\mathrm{near}}+R_{\mathrm{far}}.
\]

\emph{Near neighbors.}
If $\Delta_i\le T/n$ then certainly $r_i^n\le 1$, hence
\[
R_{\mathrm{near}}\ \le\ \#\{\,i:\ \Delta_i\le T/n\,\}
=:N_d(T/n).
\]
Since $\lambda_{\min}\ge0$, the event $\{\lambda_{i}-\lambda_{\min}\le T/n\}$ implies $\lambda_{i}\le T/n$, hence
\[
N_d(T/n)\ \le\ \sum_{i=1}^d \mathbf 1\{\lambda_i\le T/n\},\qquad
\mathbb E\,N_d(T/n)\ \le\ d\,F(T/n)\ \le\ C\,d\,(T/n)^{\beta}
\]
for a constant $C$ depending only on $(a,\beta)$.
By Markov's inequality, for any $\varepsilon>0$,
\[
\mathbb P\!\big(N_d(T/n)\ge \varepsilon\big)\ \le\ \frac{C}{\varepsilon}\,d\,(T/n)^{\beta}.
\]

\emph{Far neighbors.}
On the event $\{\lambda_{\min}\le \min\{x_0/2,\,1/(2\eta)\}\}$ (which has probability $\to1$),
we have $1-\eta \lambda_{\min}\ge \tfrac12$. Then for any $i$:

\begin{itemize}
\item If $\lambda_{i}\le 1/\eta$ and $\Delta_i> T/n$, the map $x\mapsto (1-\eta x)$ is decreasing on $[0,1/\eta]$, hence
\[
r_i=\frac{1-\eta \lambda_{i}}{1-\eta \lambda_{\min}}
=1-\frac{\eta\,(\lambda_{i}-\lambda_{\min})}{1-\eta \lambda_{\min}}
\ \le\ \exp\!\Big(-\frac{\eta\,\Delta_i}{1-\eta \lambda_{\min}}\Big)
\ \le\ \exp(-2\eta\,\Delta_i),
\]
whence $r_i^{\,n}\le \exp(-2\eta n\,\Delta_i)\le e^{-2\eta T}$.

\item If $\lambda_{i}\ge 1/\eta$, then $r_i\le \frac{|1-\eta|}{1-\eta \lambda_{\min}}\le 2|1-\eta|<1$, so $r_i^n
\le (2|1-\eta|)^n\le e^{-c_\eta n}$ with $c_\eta>0$.
\end{itemize}

Thus, on $\{\lambda_{\min}\le \min\{x_0/2,1/(2\eta)\}\}$,
\[
R_{\mathrm{far}}\ \le\ (d-1)\,e^{-2\eta T}\ +\ d\,e^{-c_\eta n}.
\]

Combining near and far parts,
\[
\boxed{\ R_{d,n}\ \le\ N_d(T/n)\ +\ (d-1)e^{-2\eta T}\ +\ d\,e^{-c_\eta n}\ } \qquad\text{with probability }1-o(1).
\]

\medskip
\noindent\textit{Step 2: supercritical regime $n^{\beta}/d\to\infty$.}
Fix $\varepsilon>0$. Choose $T=T(\varepsilon)$ so large that $e^{-2\eta T}\le \varepsilon^2/d$ for all $d$,
and note that $e^{-c_\eta n}\to0$. Then
\[
\mathbb P\!\big(R_{d,n}>\varepsilon\big)
\ \le\ \mathbb P\!\big(N_d(T/n)>\varepsilon/2\big)\ +\ o(1)
\ \le\ \frac{2C}{\varepsilon}\,d\,(T/n)^{\beta}\ +\ o(1)
\ \xrightarrow[d\to\infty]{}\ 0,
\]
because $d\,(T/n)^{\beta}\to0$ when $n^{\beta}/d\to\infty$.
Therefore $R_{d,n}\xrightarrow{\mathbb P}0$, which proves part (i):
\[
S_d(\eta,n)
=\frac{1}{d}\,(1-\eta \lambda_{\min})^{n}\,(1+o_{\mathbb P}(1)).
\]

\medskip
\noindent\textit{Step 3: critical regime $n^{\beta}/d\to\kappa\in(0,\infty)$.}
At this scale, the \emph{expected} near-edge count in a window of width ${\sim}T/n$ stays of order one:
\[
\mathbb E\,N_d(T/n)\ \le\ C\,d\,(T/n)^{\beta}\ \longrightarrow\ C\,\kappa^{-1}\,T^{\beta}.
\]
Hence $R_{d,n}$ does not vanish in probability. Let $Z_d:=n\,\lambda_{\min}$. For $\alpha>0$,
\[
\mathbb P(Z_d>\alpha)
=\mathbb P\!\big(\lambda_{\min}>\alpha/n\big)
=\bigl(1-F(\alpha/n)\bigr)^d
\ \to\ \exp\!\big(-a\,\kappa^{-1}\,\alpha^{\beta}\big),
\]
so $Z_d$ converges in distribution to a Weibull law on $[0,\infty)$. This proves part (ii).
\end{proof}

\subsection{Finite-\texorpdfstring{$d$}{d} correction at a shifted edge}

The main theorems (and Lemma~\ref{lem:left-edge-laplace}) describe the \emph{subcritical} power-law window where the spectrum can be treated as continuous near the left edge.
In practice, however, one often fits a \emph{shifted} edge law
\[
\rho((\lambda_0,\lambda])\ \approx\ Q\,(\lambda-\lambda_0)^\beta
\qquad (\lambda\downarrow\lambda_0)
\]
with $\lambda_0$ slightly \emph{below} the smallest observed eigenvalue $\lambda_{\min}$ of the finite-$d$ matrix.
When $\beta\gtrsim 1$ and/or $\lambda_0$ is not tiny, the resulting gap $x_{\min}:=\lambda_{\min}-\lambda_0$ can be non-negligible on the time scales of interest and the naive replacement $x_{\min}=0$ can substantially overpredict the trace.
The following ``gap+atom'' correction captures this finite-$d$ effect in closed form and interpolates smoothly between the power-law window and the minimum-eigenvalue dominated regime of Theorem~\ref{thm:regimes-appendix}.

\begin{proposition}[Shifted-edge finite-$d$ predictor]\label{prop:shifted-edge-finite-d}
Assume that for some $\lambda_0\ge 0$, $Q>0$, $\beta>0$, and $x_0>0$,
\[
\rho((\lambda_0,\lambda_0+x])\ =\ Q\,x^\beta\,(1+o(1))\qquad (x\downarrow 0),
\]
and let $\lambda_{\min}$ be the smallest eigenvalue with multiplicity fraction $p_{\min}:=\mathrm{mult}(\lambda_{\min})/d$.
Define the gap $x_{\min}:=\lambda_{\min}-\lambda_0\ge 0$ and $z:=2\eta n\,x_{\min}$.
Then the squared-dynamics trace admits the practical approximation
\begin{equation}\label{eq:shifted-edge-finite-d-test}
\frac1d\,\Tr\bigl((I-\eta H)^{2n}\bigr)
\ \approx\ 
e^{-2\eta\lambda_0 n}
\Bigl[
Q\,\beta\,(2\eta n)^{-\beta}\,\Gamma(\beta,z)\;+\;p_{\min}\,e^{-z}
\Bigr],
\end{equation}
where $\Gamma(\beta,z):=\int_z^\infty t^{\beta-1}e^{-t}\,dt$ is the upper incomplete Gamma function.

Moreover, for the normalized (Kaczmarz) IID row update in the $m=d$ setting, the corresponding training-loss trace satisfies
\begin{align*} %
\frac1d\,\E\Tr(HS_n)\ &\approx\ \frac{2}{2-\eta}\,e^{-2\eta\lambda_0 n}\,
Q\,\beta\Bigl[
\lambda_0\,(2\eta n)^{-\beta}\Gamma(\beta,z)\;+\;(2\eta n)^{-(\beta+1)}\Gamma(\beta+1,z)
\Bigr]
\\&\hspace{4em}+\;\frac{2}{2-\eta}\,p_{\min}\,\lambda_{\min}\,e^{-2\eta\lambda_{\min}n}.
\end{align*}
\end{proposition}

\begin{proof}[Derivation (informal)]
For the test trace, approximate $(1-\eta\lambda)^{2n}\approx e^{-2\eta n\lambda}$ on the left edge and write $\lambda=\lambda_0+x$.
Truncating the edge density at $x_{\min}$ yields
\[
\int_{x_{\min}}^\infty e^{-2\eta n(\lambda_0+x)}\,Q\beta x^{\beta-1}\,dx
=
e^{-2\eta\lambda_0 n}\,Q\beta(2\eta n)^{-\beta}\Gamma(\beta,2\eta n x_{\min}),
\]
and we add the missing mass $p_{\min}$ as an atom at $\lambda_{\min}=\lambda_0+x_{\min}$, giving~\eqref{eq:shifted-edge-finite-d-test}.
For the training trace, apply the same truncation to $\int \lambda e^{-2\eta n\lambda}\,d\rho(\lambda)$ and use the exact trace decrement identity for normalized Kaczmarz (Section~\ref{sec:iid}), which multiplies the GD-style $\lambda$-weight by the factor $2/(2-\eta)$.
\end{proof}

\subsection{Frobenius Norm under Gaussian Rows}

While the trace results above apply to general power-law spectra, we can derive specific results for the Frobenius norm when the rows are sampled as Gaussian vectors.

\begin{theorem}[Frobenius norm under Gaussian rows]\label{thm:frob-gaussian}
Assume the left edge law $F(x)\sim a\,x^\beta$ with $\beta\in(0,1)$ and let $(u_t)$ be i.i.d.\ Gaussian with covariance $A$. With $S_{t+1}=(I-\eta u_tu_t^{\top})S_t(I-\eta u_tu_t^{\top})$ and $S_0=I$, we have
\[
\mathbb E\,\Tr(S_n)\ =\ a\,\Gamma(\beta+1)\,(2\eta n)^{-\beta}\,\bigl(1+o(1)\bigr),\qquad n\to\infty.
\]
\end{theorem}

\begin{proof}[Proof sketch]
Define $S_{t+1} = (I-\eta u_t u_t^T) S_t (I-\eta u_t u_t^T)$ with $S_0 = I$.
Then
\[
\mathbb E[S_{t+1}]
= S_t - \eta(A S_t + S_t A) + \eta^2 (2 A S_t A + \Tr(A S_t) A).
\]
We see that this is analytical in $A$, so we can diagonalize and write it in the eigenbasis of $A$.
We get
\[
s_{t+1}
= \diag(1 - 2\eta \lambda + 2\eta^2 \lambda^2) s_t + \eta^2 (\lambda^T s_t) \lambda.
\]
Define $M = \diag(1 - 2\eta \lambda + 2\eta^2 \lambda^2) + \eta^2 \lambda \lambda^T$,
then
\[
\mathbb E[\Tr(S_t)]
= \hat{1}^T M^t \hat{1}.
\]

Since we assume the small eigenvalues are dominating, we would like to show that the $1-2\eta \lambda$ term dominates.
We write the generating function:
\[
\sum_{t=0}^\infty \mathbb E[\Tr(S_t)] z^t
= \hat{1}^T
(I-zM)^{-1}
\hat{1}
= \hat{1}^T
(I-zD+z\eta^2\lambda\lambda^T)^{-1}
\hat{1},
\]
where $D=\diag(1 - 2\eta \lambda + 2\eta^2 \lambda^2)$.

Using the Sherman-Morrison formula, we get
\begin{align*}
\sum_t \mathbb E[\Tr(S_t)] z^t
&= \hat{1}^T (I-zD)^{-1} \hat{1}
+ \frac{z\eta^2 (\hat{1}^T (I-zD)^{-1} \lambda)^2}{1 - z\eta^2 \lambda^T (I-zD)^{-1} \lambda}
\\&=
h_0(z) - \frac{z\eta^2 h_1(z)^2}{1 + z\eta^2 h_2(z)},
\end{align*}
where $h_k(z) = \sum_{i=1}^d \frac{\lambda_i^k}{1-z \alpha_\eta(\lambda_i)}$
and $\alpha_\eta(\lambda) = 1-2\eta \lambda + 2\eta^2 \lambda^2$.

We move to the density limit with edge density $a\beta x^{\beta-1}$.
The integral
\[
\int_0^1 \frac{x^k}{1-z \alpha_\eta(x)} a\beta x^{\beta-1} dx
\]
has singularities where $z=1/\alpha_\eta(x)$, the smallest of them being at $z=1$ since
$\alpha_\eta(x) = (1-\eta x)^2 + \eta^2 x^2 \in [1, 1-2\eta + 2\eta^2]$ for $x\in[0,1]$.

For $\beta\in(0,1)$ and $\eta\in(0,2)$ we get
\begin{align*}
H_0(z) &= (2\eta)^{-\beta} \Gamma(\beta+1) \Gamma(1-\beta) (1-z)^{\beta-1}
+ \beta (-\log(1-z)) + O(1)
\\
H_k(z) &= \beta (-\log(1-z)) + O(1) \quad k\ge 1
\\
\frac{z\eta^2 H_1(z)^2}{1 + z\eta^2 H_2(z)}
&= \beta (-\log(1-z)) + O(1).
\end{align*}

Hence the dominant term is $H_0(z)$, and extracting the coefficient of $z^n$ gives
\[
\mathbb E[\Tr(S_n)]
= a\, \Gamma(\beta+1) (2\eta)^{-\beta} n^{-\beta} (1+o(1)).
\]

For the Marchenko-Pastur case we have $a=2/\pi$ and $\beta=1/2$.
We have $\Gamma(3/2) = \sqrt{\pi}/2$, so
\[
\mathbb E[\Tr(S_n)]
= \frac{2}{\pi} \cdot \frac{\sqrt{\pi}}{2} \cdot (2\eta)^{-1/2} n^{-1/2} (1+o(1))
= \frac{1}{\sqrt{2 \pi \eta}} n^{-1/2} (1+o(1)).
\]
This matches the flip-flop results and simulations for uniform row sampling when $\eta=1$.
\end{proof}

\subsection{Connection to Row-Kaczmarz Steps}

The above theorems are stated in terms of epochs $n$ (one epoch is a full pass through the dataset, i.e.\ $m$ single-row updates).
Therefore, if $k$ denotes the total number of row-Kaczmarz steps, we have $n \approx k/m$ epochs.
In the critically determined square case $m=d$ (used throughout our experiments and Gaussian calculations), this coincides with $n \approx k/d$.

\begin{corollary}[Translation to row steps]
Under the assumptions of the main theorems, with $k$ row-Kaczmarz steps and $n = \lfloor k/m \rfloor$ epochs:
\begin{itemize}
\item The power-law window occurs when $k^{\beta} = o(d\,m^{\beta})$ (equivalently, $(k/m)^\beta=o(d)$)
\item The critical scale is $k^{\beta} \sim \kappa\, d\, m^{\beta}$ (equivalently, $(k/m)^\beta\sim \kappa d$)
\item The supercritical regime is $k^{\beta}/(d\,m^{\beta}) \to \infty$
\item The test loss scales as $(\eta\,k/m)^{-\beta}$ in the power-law window
\end{itemize}
\end{corollary}

This corollary shows that the scaling behavior is consistent whether we count in epochs or individual row updates, with appropriate adjustments for the dimension $d$.

\section{IID Row Sampling Proofs}
\label{app:iid-proofs}

This appendix proves Theorem~\ref{thm:iid-sgd}.
We work in the homogeneous setting $y=0$, since replacing $w_t$ by $w_t-w^*$ reduces the general least-squares problem to this case.

\subsection{Second-moment recursion}
\label{app:iid:recursion}

Recall the IID row-sampling update
\[
w_{t+1} = (I-\eta\,u_tu_t^\top)\,w_t,
\qquad u_t \in \R^d\ \text{iid},
\]
and the associated second-moment matrices
\[
S_0:=I,
\qquad
S_{t+1}:=(I-\eta\,u_tu_t^\top)\,S_t\,(I-\eta\,u_tu_t^\top).
\]
Let $M_t:=\E[S_t]$. Conditioning on $u_t$ and expanding gives the deterministic recursion
\begin{equation}
M_{t+1}
= M_t - \eta\,(A M_t + M_t A) + \eta^2\,\mathcal B(M_t),
\label{eq:app-iid-recursion}
\end{equation}
where $A:=\E[u_tu_t^\top]$ and
\begin{equation}
\mathcal B(M):=\E\!\left[u_tu_t^\top\,M\,u_tu_t^\top\right].
\label{eq:app-iid-B}
\end{equation}
The operator $\mathcal B$ depends on fourth moments of the row distribution; in particular, for IID sampling from a \emph{fixed} dataset $X$ with $u_t$ uniform over the rows, $\mathcal B$ encodes empirical fourth moments of~$X$.
\paragraph{Relation to the Gram matrix $H$.}
In the main paper we write $H:=X^\top X/m$ for the (empirical) Gram matrix of the unnormalized rows.
If we row-normalize so that $\|x_i\|^2=d$, then $u_i=x_i/\|x_i\|$ implies
$A=\E[u_tu_t^\top]=\frac{1}{m}\sum_i u_i u_i^\top = H/d$.

\subsection{A bounded-kurtosis (near-Wick) edge condition}
\label{app:iid:wick}

Write the eigendecomposition $A=V\Lambda V^\top$, with $\Lambda=\diag(\lambda_1,\dots,\lambda_d)$.
For a fixed dataset, define the rotated rows $y_i:=V^\top x_i$ and the edge-whitened coordinates $\xi_{ij}:=y_{ij}/\sqrt{\lambda_j}$ (for $\lambda_j>0$).
The next condition is a sufficient hypothesis under which the recursion \eqref{eq:app-iid-recursion} closes (approximately) in terms of $\Lambda$ on the spectral edge.

\begin{definition}[Near-Wick edge condition]
\label{def:edge-wick}
Fix a window scale $\delta=\delta_{d,n}>0$ (typically $\delta\asymp 1/n$).
We say the dataset satisfies a \emph{near-Wick} condition on $[0,\delta]$ if there exist constants $\kappa_4,\tau=O(1)$ and errors $\varepsilon_d\to0$ such that for all indices $j,k$ with $\lambda_j,\lambda_k\le\delta$,
\begin{equation}
\bigg|
\frac{1}{m}\sum_{i=1}^m y_{ij}^2y_{ik}^2
\;-\;\Big(\tau\,\lambda_j\lambda_k+(\kappa_4-\tau)\lambda_j^2\mathbf 1\{j=k\}\Big)
\bigg|
\ \le\ \varepsilon_d\,\lambda_j\lambda_k,
\label{eq:edge-wick-diagonal}
\end{equation}
and the mixed fourth moments are negligible in the sense that
\begin{equation}
\frac{1}{m}\sum_{i=1}^m y_{ij}y_{ik}y_{i\ell}y_{ir}
\;=\;O(\varepsilon_d)\,\sqrt{\lambda_j\lambda_k\lambda_\ell\lambda_r}
\qquad\text{whenever } \{j,k\}\neq\{\ell,r\}.
\label{eq:edge-wick-offdiag}
\end{equation}
\end{definition}

\subsection{Reduction to a diagonal-plus-rank-one recursion}
\label{app:iid:closure}

Under Definition~\ref{def:edge-wick}, the operator $\mathcal B$ in \eqref{eq:app-iid-B} acts on (approximately) diagonal matrices as a diagonal-plus-rank-one map on the edge window.
Concretely, writing $M=\diag(s)$ in the eigenbasis of $A$ and using \eqref{eq:edge-wick-diagonal}, for $\lambda_j\le\delta$ we have
\begin{align}
(\mathcal B(M))_{jj}
&=\sum_{k=1}^d\Big(\frac1m\sum_{i=1}^m y_{ij}^2y_{ik}^2\Big)s_k \notag\\
&=\tau\,\lambda_j(\lambda^\top s)+(\kappa_4-\tau)\lambda_j^2 s_j + O(\varepsilon_d)\,\lambda_j(\lambda^\top s).
\label{eq:app-iid-Bjj}
\end{align}
Moreover, the off-diagonal control \eqref{eq:edge-wick-offdiag} prevents the recursion \eqref{eq:app-iid-recursion} from generating appreciable mass in off-diagonal entries on the edge window when started from $M_0=I$.
Thus, on the left-edge window that drives the $n^{-\beta}$ regime, we may reduce \eqref{eq:app-iid-recursion} to the effective diagonal recursion
\begin{equation}
s_{t+1}
=\diag\!\bigl(\alpha_\eta(\lambda)\bigr)s_t
\;+\;\tau\eta^2(\lambda^\top s_t)\lambda
\;+\;o(1)\cdot \lambda(\lambda^\top s_t),
\label{eq:app-iid-diag-recursion}
\end{equation}
where $\alpha_\eta(\lambda):=1-2\eta\lambda+(\kappa_4-\tau)\eta^2\lambda^2$ and the $o(1)$ error is uniform over $\lambda\le\delta$.
The key point is that $\alpha_\eta(\lambda)=1-2\eta\lambda+O(\lambda^2)$ as $\lambda\downarrow0$, with bounded $O(\lambda^2)$ coefficient whenever $\kappa_4,\tau=O(1)$.

\subsection{Leading asymptotics}
\label{app:iid:asy}

We now sketch how \eqref{eq:iid-test-asy}--\eqref{eq:iid-train-asy} follow from \eqref{eq:app-iid-diag-recursion}.
Ignoring the $o(1)$ perturbation (which contributes only a multiplicative $1+o(1)$ factor in the edge-dominated regime), \eqref{eq:app-iid-diag-recursion} is a diagonal operator plus a rank-one coupling.
Writing the trace generating function
\[
G(z):=\sum_{t\ge 0} z^t\,\Tr(M_t),
\]
a Sherman--Morrison resolvent calculation (identical to the Gaussian/elliptical case) yields
\begin{equation}
G(z)=h_0(z)\;-\;\frac{\tau\eta^2 z\,h_1(z)^2}{1+\tau\eta^2 z\,h_2(z)}\;+\;\text{(edge-}o(1)\text{)},
\label{eq:app-iid-genfun-test}
\end{equation}
where
\[
h_k(z):=\sum_{i=1}^d \frac{\lambda_i^k}{1-z\,\alpha_\eta(\lambda_i)}.
\]
The leading term in \eqref{eq:app-iid-genfun-test} is $h_0(z)$: under the left-edge law with $\beta\in(0,1)$, $h_0(z)$ has the strongest singularity at $z=1$, while the correction term is at most logarithmically singular and hence contributes $o(t^{-\beta})$ to coefficients.
Coefficient extraction for $h_0$ is explicit:
\[
[z^t]\,h_0(z)=\sum_{i=1}^d \alpha_\eta(\lambda_i)^t.
\]
Since $\alpha_\eta(\lambda)=1-2\eta\lambda+O(\lambda^2)$, the same Laplace/edge-mass estimate as in Appendix~\ref{app:powerlaw-proofs} (with $2\eta$ in place of $\eta$) gives
\[
\frac{1}{d}\sum_{i=1}^d \alpha_\eta(\lambda_i)^t
=Q\,\Gamma(\beta+1)\,(2\eta t)^{-\beta}\,(1+o(1)).
\]
To match the epoch-time statements in the main text, note that in the row-normalized setting $A=H/d$, so the eigenvalues of $A$ are $\lambda_i(H)/d$ and the edge constant rescales accordingly.
With one epoch equal to $m=d$ row updates (so $t=dn$), this yields \eqref{eq:iid-test-asy}.

\medskip
\noindent
For the training trace in the normalized (Kaczmarz) model, we use a different (and simpler) route.
Since $\|u_t\|=1$, we have the exact identity
\[
\Tr(S_{t+1})=\Tr(S_t)-\eta(2-\eta)\,u_t^\top S_t u_t.
\]
Taking expectations and using $\E[u_tu_t^\top]=A$ yields
\[
\Tr(M_{t+1})-\Tr(M_t)=-\eta(2-\eta)\,\Tr(AM_t).
\]
In the row-normalized setting where $A=H/d$, this becomes
\[
\frac{1}{d}\Tr(HM_t)= -\frac{1}{\eta(2-\eta)}\cdot \frac{1}{d}\bigl(\Tr(M_{t+1})-\Tr(M_t)\bigr).
\]
Thus, in epoch units (with $t=mn$), the training loss is proportional to a discrete derivative of the test loss.
Applying this relation to the test-loss asymptotic \eqref{eq:iid-test-asy} gives the extra factor of $\beta$ and yields \eqref{eq:iid-train-asy}.

\section{Free Probability Derivations}
\label{app:freeprop}

This appendix provides complete derivations of the moment generating functions and S-transforms used in Section~\ref{sec:shuffling}.
We work in the framework of free probability theory, which provides the appropriate tools for analyzing products of freely independent random matrices in the large-dimensional limit.

\subsection{Background on Free Probability}

Let $(\mathcal{A}, \phi)$ be a non-commutative probability space, where $\mathcal{A}$ is a unital algebra and $\phi: \mathcal{A} \to \mathbb{C}$ is a linear functional with $\phi(1) = 1$.
Random variables in this space are elements of $\mathcal{A}$, and their ``moments'' are values of $\phi$ on their powers.

\begin{definition}[Free Independence]
Random variables $X_1, \ldots, X_m \in \mathcal{A}$ are \emph{freely independent} if for any polynomials $p_1, \ldots, p_k$ and any indices $i_1, \ldots, i_k$ with $i_j \neq i_{j+1}$:
\[
\phi\bigl(p_1(X_{i_1}) \cdots p_k(X_{i_k})\bigr) = 0
\]
whenever $\phi(p_j(X_{i_j})) = 0$ for all $j$.
\end{definition}

The key tool for multiplying freely independent random variables is the S-transform.

\begin{definition}[S-transform]
For a random variable $X$ with $\phi(X) \neq 0$, define the moment series
\[
\varphi_X(z) = \sum_{n=1}^{\infty} \phi(X^n) z^n.
\]
The S-transform is
\[
S_X(z) = \frac{z + 1}{z} \cdot \varphi_X^{-1}(z),
\]
where $\varphi_X^{-1}$ denotes the compositional inverse (i.e., $\varphi_X(\varphi_X^{-1}(z)) = z$).
\end{definition}

\begin{theorem}[S-transform multiplicativity]
If $X$ and $Y$ are freely independent with $\phi(X), \phi(Y) \neq 0$, then
\[
S_{XY}(z) = S_X(z) \cdot S_Y(z).
\]
\end{theorem}

This multiplicativity is what makes free probability powerful for analyzing products of random matrices.

\subsection{Asymptotic freeness for the projection model}
\label{sec:app:freeprop:freeness}

\paragraph{Lambert $W$ conventions (branches vs.\ parameters).}
We write $W_0$ and $W_{-1}$ for the two real branches of the classical Lambert $W$ function (the inverse of $u\mapsto u e^{u}$ on $[-e^{-1},0)$),
with $W_0(\cdot)\in[-1,0)$ the principal branch and $W_{-1}(\cdot)\in(-\infty,-1]$ the lower branch.
Separately, we write $W_r$ for the \emph{$r$-Lambert} function (the inverse of $u\mapsto u e^{u}+r u$), where the subscript $r$ is a parameter and should not be confused with a branch index.

\paragraph{Why it is legitimate to replace $(P_{d,i,b})$ by free projections.}
In the shuffle-once derivation below we study products of random projections
\[
P_{d,i,b} \;=\; I - X_i (X_i^\top X_i)^{-1} X_i^\top,
\]
where $X_i\in\R^{d\times b}$ have i.i.d.\ $\mathcal N(0,1)$ entries and are independent across $i$.
Each $P_{d,i,b}$ is the orthogonal projector onto a uniformly random $(d-b)$-dimensional subspace, and its law is invariant under conjugation by Haar orthogonal matrices.
Equivalently, there exist i.i.d.\ Haar orthogonal matrices $(U_i)$ such that
\[
P_{d,i,b}\ \stackrel{d}{=}\ U_i
\begin{pmatrix}
I_{d-b} & 0\\
0 & 0_b
\end{pmatrix}
U_i^\top .
\]
Voiculescu's asymptotic freeness theorem implies that independent orthogonally (or unitarily) invariant random matrices are \emph{asymptotically free} in normalized trace as $d\to\infty$.
Concretely, for any fixed $m$ and any noncommutative $^*$-polynomial $p$ in $m$ variables,
\[
\frac1d\Tr\,p(P_{d,1,b},\dots,P_{d,m,b})
\ \xrightarrow[d\to\infty]{\mathbb P}\
\phi\bigl(p(P_1,\dots,P_m)\bigr),
\]
where $(P_1,\dots,P_m)$ are freely independent projections with $\phi(P_i)=\lim_{d\to\infty}\frac1d\Tr(P_{d,i,b})=1-b/d$.
Moreover, normalized traces of such polynomials concentrate (fluctuations are of order $d^{-2}$ under standard moment assumptions), so the limiting moments are deterministic; see \citet{voiculescu1991limit} and the textbook treatments \citet{voiculescu1992free, nica2006lectures, mingo2017free}.
In the shuffle-once calculation we first take $d\to\infty$ for fixed $m$ (so freeness applies to a finite family), and then send $m\to\infty$ with $\rho$ fixed; this iterated limit yields the closed form S-transform used below.
This is the sense in which the ``free probability limit'' used below justifies replacing the concrete projection matrices by free projections with matching trace.

\subsection{Derivation for Shuffle-Once (Theorem~\ref{thm:shuffle-once})}

Consider a sequence of independent random Gaussian matrices $X_i \in \mathbb{R}^{d \times b}$ with i.i.d.\ $\mathcal{N}(0,1)$ entries.
Define the projection matrices
\[
P_{d,i,b} = I - X_i (X_i^\top X_i)^{-1} X_i^\top,
\]
and the epoch product
\[
A_{d,m,\lceil \rho d / m \rceil} = \prod_{i=1}^m P_{d,i,\lceil \rho d / m \rceil}.
\]
This represents the Kaczmarz iteration with batch size $b = \lceil \rho d / m \rceil$.

\paragraph{Step 1: Moments of a single projection.}
Since $P_{d,i,b}$ is a projection onto the orthogonal complement of $\mathrm{span}(X_i)$, we have $P_{d,i,b}^2 = P_{d,i,b}$ and
\[
\lim_{d \to \infty} \frac{1}{d} \mathrm{Tr}(P_{d,i,\lceil \rho d / m \rceil}^n)
= \lim_{d \to \infty} \frac{1}{d} \mathrm{Tr}(P_{d,i,\lceil \rho d / m \rceil})
= \lim_{d \to \infty} \frac{1}{d} \left( d - \lceil \rho d / m \rceil \right)
= 1 - \frac{\rho}{m}.
\]
Thus, by the asymptotic freeness discussion in Section~\ref{sec:app:freeprop:freeness}, we may model the family as freely independent projections $(P_i)$ with $P_i^2=P_i$ and $\phi(P_i)=1-\rho/m$ (hence $\phi(P_i^n)=1-\rho/m$ for all $n\ge1$).

\paragraph{Step 2: S-transform of a projection.}
The moment series for $P$ is
\[
\varphi_P(z) = \sum_{n=1}^{\infty} \left( 1 - \frac{\rho}{m} \right) z^n
= \left( 1 - \frac{\rho}{m} \right) \frac{z}{1-z}.
\]
To find $\varphi_P^{-1}$, solve $x = (1 - \rho/m) \frac{y}{1-y}$ for $y$:
\begin{align*}
x(1-y) &= \left(1 - \frac{\rho}{m}\right) y \\
x &= \left(1 - \frac{\rho}{m} + x\right) y \\
y &= \frac{x}{x + 1 - \rho/m}.
\end{align*}
Thus
\[
\varphi_P^{-1}(z) = \frac{mz}{mz + m - \rho}.
\]
The S-transform is
\[
S_P(z) = \frac{z + 1}{z} \cdot \varphi_P^{-1}(z)
= \frac{z+1}{z} \cdot \frac{mz}{mz + m - \rho}
= \frac{m(z+1)}{mz + m - \rho}
= 1 + \frac{\rho}{mz + m - \rho}.
\]

\paragraph{Step 3: S-transform of the epoch product.}
Since $A_m = \prod_{i=1}^m P_i$ is a product of $m$ freely independent copies of $P$, the S-transform multiplicativity gives
\[
S_{A_m}(z) = \bigl(S_P(z)\bigr)^m = \left( 1 + \frac{\rho}{m(z + 1) - \rho} \right)^m.
\]
In the limit $m \to \infty$ with $\rho$ fixed:
\[
S_A(z) = \lim_{m \to \infty} \left( 1 + \frac{\rho}{m(z + 1) - \rho} \right)^m
= \exp\left( \frac{\rho}{1 + z} \right).
\]

\paragraph{Step 4: Inverting to get the moment generating function.}
From the S-transform relation:
\[
\varphi_A^{-1}(z) = \frac{z}{z + 1} S_A(z) = \frac{z}{z + 1} \exp\left( \frac{\rho}{1 + z} \right).
\]
To invert this, we need to solve $x = \frac{y}{y+1} \exp\left(\frac{\rho}{1+y}\right)$ for $y$ in terms of $x$.

Multiply both sides by $\exp(-\rho)$:
\[
x \exp(-\rho) = \frac{y}{y + 1} \exp\left( \frac{\rho}{1 + y} - \rho \right)
= \frac{y}{y + 1} \exp\left( \frac{-\rho y}{1 + y} \right).
\]
Multiply by $-\rho$:
\[
-\rho x \exp(-\rho) = \frac{-\rho y}{y + 1} \exp\left( \frac{-\rho y}{1 + y} \right).
\]
The right side has the form $u \exp(u)$ with $u = \frac{-\rho y}{1 + y}$.
Applying the Lambert $W$ function (the inverse of $u \mapsto u e^u$):
\[
\frac{-\rho y}{1 + y} = W_0\bigl( -\rho x \exp(-\rho) \bigr).
\]
Solving for $y$:
\begin{align*}
-\rho y &= W_0\bigl( -\rho x \exp(-\rho) \bigr) (1 + y) \\
-\rho y - y \cdot W_0(\cdot) &= W_0(\cdot) \\
y &= \frac{-W_0\bigl( -\rho x \exp(-\rho) \bigr)}{\rho + W_0\bigl( -\rho x \exp(-\rho) \bigr)}.
\end{align*}
Setting $a = -\rho \exp(-\rho)$, the generating function (including the $n=0$ term) is
\begin{equation}
\label{eq:mgf-shuffle}
\sum_{n=0}^{\infty} \phi(A^n) z^n = \frac{\rho}{\rho + W_0(az)}.
\end{equation}

\paragraph{Step 5: Extracting moments.}
Near $z = 0$, expand
\[
\sum_{n=0}^{\infty} \phi(A^n) z^n = \frac{1}{1 + W_0(az)/\rho}
= \sum_{r=0}^{\infty} \left( -\frac{W_0(az)}{\rho} \right)^r.
\]
Using the Taylor series for powers of the Lambert $W$ function:
\[
\bigl( W_0(x) \bigr)^r = \sum_{n=r}^{\infty} \frac{-r (-n)^{n-r-1}}{(n-r)!} x^n,
\]
and extracting the coefficient of $z^n$, we obtain after algebraic manipulation:
\begin{equation}
\label{eq:moment-formula}
\phi(A^n) = \frac{(n \rho)^n}{n! \exp(n \rho)} + (1 - \rho) Q(n, n\rho),
\end{equation}
where $Q(n, n\rho) = \frac{\Gamma(n, n\rho)}{\Gamma(n)}$ is the regularized incomplete gamma function.

\subsection{Moment Equivalence for Projections}

A key technical result is that for products of freely independent projections, several related quantities have identical moments.

\begin{theorem}[Moment equivalence]
\label{thm:moment-equiv}
Let $P_1, P_2, \ldots, P_m$ be freely independent projections (not necessarily identically distributed), i.e., $P_i^2 = P_i$.
Let $A = P_1 P_2 \cdots P_m$ and $A^\top = P_m P_{m-1} \cdots P_1$.
Then for all $n \geq 1$:
\[
\phi(A^n) = \phi((A^\top)^n) = \phi((A A^\top)^n) = \phi((A^\top A)^n).
\]
\end{theorem}

\begin{proof}
We prove this by induction on $m$.

\textbf{Base case ($m=1$):} For $m = 1$, we have $A = A^\top = AA^\top = A^\top A = P_1$, so all quantities are equal.

\textbf{Inductive step:}
Suppose the statement holds for $m$ projections.
Let $P = P_{m+1}$, $B = P_1 P_2 \cdots P_m$, and $B^\top = P_m P_{m-1} \cdots P_1$, so that $A = BP$ and $A^\top = PB^\top$.

Using the trace property $\phi(XY) = \phi(YX)$ and the projection property $P^2 = P$:
\begin{align*}
\phi(A^n) &= \phi((BP)^n), \\
\phi((A^\top)^n) &= \phi((PB^\top)^n) = \phi((B^\top P)^n), \\
\phi((AA^\top)^n) &= \phi((BPB^\top)^n) = \phi((B^\top BP)^n), \\
\phi((A^\top A)^n) &= \phi((PB^\top BP)^n) = \phi((B^\top BP)^n).
\end{align*}

Since $P$ is freely independent from $B$, $B^\top$, and $B^\top B$, the S-transform multiplicativity gives:
\begin{align*}
S_{BP}(z) &= S_B(z) S_P(z), \\
S_{B^\top P}(z) &= S_{B^\top}(z) S_P(z), \\
S_{B^\top B P}(z) &= S_{B^\top B}(z) S_P(z).
\end{align*}
By the inductive hypothesis, $B$, $B^\top$, and $B^\top B$ have identical moments, hence identical S-transforms.
Therefore all the above S-transforms are equal, implying equal moments.
\end{proof}

\begin{corollary}
The squared singular values of $A_m = \prod_{i=1}^m P_i$ have the same moments as its eigenvalues in the free probability limit.
\end{corollary}

\subsection{Extension to General Step Size (SGD)}
\label{sec:app-sgd}

For SGD with step size $\eta \neq 1$, the projection matrices are replaced by
\[
P_{d,i,b} = I - \eta X_i (X_i^\top X_i)^{-1} X_i^\top.
\]
In the free probability limit, $P_i$ now takes the value $1$ with probability $1 - \rho/m$ (when the row is not selected) and $1 - \eta$ with probability $\rho/m$ (when selected).

\paragraph{Step 1: Moment series.}
The moment series becomes
\[
\varphi_P(z) = \sum_{n=1}^{\infty} \left[ \left(1 - \frac{\rho}{m}\right) + \frac{\rho}{m}(1-\eta)^n \right] z^n
= \left(1 - \frac{\rho}{m}\right) \frac{z}{1-z} + \frac{\rho}{m} \cdot \frac{(1-\eta)z}{1-(1-\eta)z}.
\]

\paragraph{Step 2: S-transform in the large-$m$ limit.}
After computing $\varphi_P^{-1}$ (which involves a quadratic formula) and taking the limit $m \to \infty$:
\[
S_A(z) = \lim_{m \to \infty} \bigl(S_P(z)\bigr)^m = \exp\left( \frac{\eta \rho}{\eta z + 1} \right).
\]

\paragraph{Step 3: Inverting via generalized Lambert $W$.}
The compositional inverse involves the generalized Lambert $W$ function.
Following the derivation in \citet{mezHo2017generalization}, we obtain
\[
\varphi_A(z) + 1 = \frac{\rho}{\eta \rho + W_{r(z)}\bigl(-\eta^2 \rho z \exp(-\eta \rho)\bigr)} - \frac{1 - \eta}{\eta},
\]
where $r(z) = -(1-\eta)z\exp(-\eta\rho)$ and $W_r$ denotes the $r$-Lambert function (the inverse of $x\exp(x) + rx$).

\paragraph{Step 4: Asymptotic analysis.}
The generating function has a singularity at $z = 1$ of order $1/2$.
Specifically, for $\rho = 1$:
\[
\lim_{z \to 1^-} \frac{\sqrt{1-z}}{\eta + W_{r(z)}(-\eta^2 z \exp(-\eta))}
= \sqrt{\frac{2-\eta}{\eta}}.
\]
By singularity analysis, the $n$-th moment decays as
\[
\phi(A^n) \sim \frac{1}{\Gamma(1/2)} \cdot n^{-1/2} \cdot \sqrt{\frac{2-\eta}{\eta}}
= \frac{1}{\sqrt{\pi n}} \cdot \sqrt{\frac{2-\eta}{\eta}}.
\]
For $\eta = 1$ (Kaczmarz), this gives $\phi(A^n) \sim \frac{1}{\sqrt{\pi n}}$, matching Theorem~\ref{thm:shuffle-once}.

\subsection{Stieltjes Transform and Eigenvalue Distribution}
\label{sec:stieltjes-cdf}

The moment generating function~\eqref{eq:mgf-shuffle} can be used to derive the eigenvalue distribution of $A$ via the Stieltjes transform.

\begin{proposition}
The Stieltjes transform of the squared singular values of $A$ is
\[
\mathbb{E}\left[\frac{1}{\sigma^2 - z}\right] = \frac{-\rho}{z(\rho + W_0(a/z))},
\]
where $a = -\rho \exp(-\rho)$.
\end{proposition}

\begin{proof}
From the moment generating function:
\[
\mathbb{E}\left[\frac{1}{1 - \sigma^2 z}\right] = \sum_{n=0}^{\infty} \mathbb{E}[\sigma^{2n}] z^n = \frac{\rho}{\rho + W_0(az)}.
\]
The Stieltjes transform is related by
\[
\mathbb{E}\left[\frac{1}{\sigma^2 - z}\right] = -\frac{1}{z} \mathbb{E}\left[\frac{1}{1 - \sigma^2/z}\right]
= -\frac{1}{z} \cdot \frac{\rho}{\rho + W_0(a/z)}.
\]
\end{proof}

The probability density function can be recovered via the Stieltjes-Perron inversion formula:
\[
p_{\sigma^2}(x) = \lim_{\epsilon \to 0^+} \frac{1}{\pi} \mathrm{Im}\left[\mathbb{E}\left[\frac{1}{\sigma^2 - (x + i\epsilon)}\right]\right].
\]
The imaginary part arises from the branch cut of the Lambert $W$ function.
Explicitly:
\[
p_{\sigma^2}(x) = \frac{1}{2\pi i} \left(
\frac{\rho}{x(\rho + W_0(a/x))} - \frac{\rho}{x(\rho + W_{-1}(a/x))}
\right),
\]
where $W_0$ and $W_{-1}$ are the two real branches of the Lambert $W$ function.

The CDF can be expressed in closed form:
\[
F_{\sigma^2}(x) = -\log W_0\left(\frac{a}{x}\right) - (\rho - 1)\log\left(W_0\left(\frac{a}{x}\right) + \rho\right) + C,
\]
where $C$ is a normalization constant.
The support is $[\exp(-2\rho), 1]$ for $\rho \leq 1$, determined by the branch cut structure of the Lambert $W$ function.

\section{Stieltjes Transform Proofs for Shuffle-Once}
\label{app:stieltjes}

This appendix provides detailed proofs of the shuffle-once convergence results using Stieltjes transform methods.
These proofs complement the free probability approach of Appendix~\ref{app:freeprop} by providing a more direct analysis that tracks the evolution of resolvents through the iteration.

\subsection{Warm-up: Trace Moments for Test Loss}
\label{sec:warmup}

As a warm-up, we derive the same trace moment result obtained via free probability, using the Stieltjes transform approach.
This establishes the key techniques that will be extended to handle the Frobenius norm (test loss) and training loss.

We define the standard Stieltjes transform:
\begin{align}
s_t(\lambda)
= \frac{1}{d} \mathbb{E}\,\mathrm{Tr}((A_t - \lambda I)^{-1})
= -\sum_{n\ge 0} \frac{1}{d} \mathbb{E}\,\mathrm{Tr}(A_t^n) \lambda^{-n-1}.
\label{eq:stieltjes-simple}
\end{align}
Our goal is to analyze the behavior as $t$ grows proportionally to $d$.
In the limit $d \to \infty$ with $t = \rho d$:
\begin{align*}
\lim_{d\to\infty} s_{\rho d}(\lambda)
= -\frac{1}{\bigl(1+\rho^{-1}W(-\rho e^{-\rho}/\lambda)\bigr)\lambda}
= -\sum_{n\ge 0}
\left(e^{-\rho n} \sum_{k=0}^n \frac{n-k}{n} \frac{ (\rho n)^{k}}{k!}\right)
\lambda^{-n-1},
\end{align*}
where $W$ denotes the principal branch of the Lambert $W$ function.
For the critical case $\rho=1$, this gives $\frac{1}{d} \mathbb{E}\,\mathrm{Tr}(A_{d}^n) \to \frac{n^n}{e^n n!} \sim \frac{1}{\sqrt{2\pi n}}$.

\begin{proof}
\textbf{Step 1: Set up the recursion.}
Recall that $A_{t+1} = (I - u_{t+1} u_{t+1}^\top) A_t$, where $u_{t+1}$ is a random vector of unit length.
Writing $u = u_{t+1}$, we expand $(A_{t+1}-\lambda I)^{-1}$ using the Sherman-Morrison formula:
\[
(A - uv^{\top})^{-1} = A^{-1} + \frac{A^{-1} u v^{\top} A^{-1}}{1 - v^{\top} A^{-1} u}
\]
and the elementary identity $(U - \lambda I)^{-1}U = I + \lambda (U - \lambda I)^{-1}$.

Setting $R = (A_t - \lambda I)^{-1}$ and expanding:
\begin{align}
(A_{t+1}-\lambda I)^{-1}
&= (A_t - \lambda I)^{-1}
- \frac{(A_t - \lambda I)^{-1} uu^\top}{\lambda u^\top (A_t - \lambda I)^{-1} u}
- \frac{(A_t - \lambda I)^{-1}uu^\top (A_t - \lambda I)^{-1}}{u^\top (A_t - \lambda I)^{-1} u}.
\label{eq:inverse1}
\end{align}
Taking the trace:
\begin{align*}
\mathrm{Tr}((A_{t+1}-\lambda I)^{-1})
&= \mathrm{Tr}(R) - \frac{1}{\lambda} - \frac{u^\top R^2 u}{u^\top R u}.
\end{align*}

\textbf{Step 2: Concentration argument.}
For large $d$, both the numerator and denominator in the ratio $\frac{u^\top R^2 u}{u^\top R u}$ concentrate around their means.
The function $f(u)=u^\top R u$ is Lipschitz with constant $2\|R\|$.
By concentration of Lipschitz functions on the sphere \citep{vershynin2018high}:
\[
|u^\top R u - \mathbb{E}[u^\top R u]| \le 2 \|R\| \sqrt{\log(2/\delta) / d}
\quad\text{with probability $1-\delta$}.
\]
Since $\|A_{t}\| \le 1$, we have $\|R\| = \|(A_t - \lambda I)^{-1}\| \le 1/(|\lambda|-1)$ for $|\lambda| > 1$.
Taking $\delta=1/d$, both numerator and denominator concentrate with $O(1/\sqrt{d})$ deviations, hence:
\[
\frac{u^\top R^2 u}{u^\top R u}
=
\frac{\mathrm{Tr}(R^2)/d + o(1)}{\mathrm{Tr}(R)/d + o(1)}
\]
uniformly in $t \le \rho d$.
Therefore:
\begin{align*}
\mathbb{E}[\mathrm{Tr}((A_{t+1}-\lambda I)^{-1})]
&= \mathbb{E}[\mathrm{Tr}(R)] - \frac{1}{\lambda} - \frac{\mathbb{E}[\mathrm{Tr}(R^2)]/d}{\mathbb{E}[\mathrm{Tr}(R)]/d} + o(1).
\end{align*}

\textbf{Step 3: Transform to differential equation.}
Writing $s(\rho, \lambda) = \frac{1}{d} \mathbb{E}\,\mathrm{Tr}((A_{\rho d} - \lambda I)^{-1})$ and noting that $\frac{\partial s}{\partial\lambda} = \frac{1}{d} \mathbb{E}\,\mathrm{Tr}((A_{\rho d} - \lambda I)^{-2})$, we obtain in the limit $d\to\infty$ the transport equation:
\begin{equation}
\frac{\partial s}{\partial \rho} = -\frac{1}{\lambda} - \frac{1}{s}\,\frac{\partial s}{\partial \lambda}.
\label{eq:pde-stieltjes}
\end{equation}

\textbf{Step 4: Solve the PDE using characteristics.}
With initial condition $s(0,\lambda)=1/(1-\lambda)$ (corresponding to $A_0=I$), we use the method of characteristics.
Consider a curve $\rho\mapsto(\rho,\lambda(\rho))$ with $\lambda'(\rho)=\frac{1}{s(\rho,\lambda(\rho))}$.
Plugging into~\eqref{eq:pde-stieltjes} gives:
\[
\frac{d}{d\rho}s(\rho,\lambda(\rho)) = -\frac{1}{\lambda(\rho)}.
\]
Using the product rule, one finds that $C:=\lambda(\rho)\,s(\rho,\lambda(\rho))$ is constant along the characteristic.
Hence $s=C/\lambda$ and $\lambda'=\lambda/C$, giving $\lambda(\rho)=\lambda_0 e^{\rho/C}$.

From $C=\lambda_0 s(0,\lambda_0) = \frac{\lambda_0}{1-\lambda_0}$, we get $\lambda(\rho)=\lambda_0 e^{\rho(1-\lambda_0)/\lambda_0}$.
Inverting via the Lambert $W$ function:
\[
C(\rho,\lambda) = -\frac{\rho}{W(-\rho e^{-\rho}/\lambda)+\rho}, \quad
\lambda_0=\frac{C}{1+C}.
\]
The solution is:
\begin{align*}
\lim_{d\to\infty} s_{\rho d}(\lambda)
= -\frac{1}{\bigl(1+\rho^{-1}W(-\rho e^{-\rho}/\lambda)\bigr)\lambda}.
\end{align*}

\textbf{Step 5: Extract the moments.}
The moments are:
\[
\lim_{d\to\infty} \frac{1}{d} \mathbb{E}\,\mathrm{Tr}(A_{\rho d}^n)
= e^{-\rho n} \sum_{k=0}^n \frac{n-k}{n} \frac{ (\rho n)^{k}}{k!},
\]
which for $\rho=1$ gives $\frac{n^n}{e^n n!}\sim \frac{1}{\sqrt{2\pi n}}$ by Stirling's approximation.
\end{proof}

\paragraph{Refined asymptotics at $\rho=1$.}
Using $m_n(1)=e^{-n}n^n/n!$ and Stirling's series:
\begin{equation}
\boxed{
m_n(1)
=\frac{1}{\sqrt{2\pi n}}
\left(
1-\frac{1}{12n}
+\frac{1}{288n^2}
+O(n^{-3})
\right).
}
\label{eq:warmup-refined}
\end{equation}

\paragraph{Phase transitions.}
The formula has a clean probabilistic interpretation. For $N \sim \text{Poisson}(\rho n)$:
\[
m_n(\rho) = \mathbb{E}\left[\left(1-\frac{N}{n}\right) \mathbf{1}_{\{N \le n\}}\right] = (1-\rho) \Pr(N \le n-1) + \Pr(N = n).
\]
This immediately yields the three-regime behavior:
\begin{itemize}
\item For $\rho < 1$: $m_n(\rho) \to 1-\rho$ as $n \to \infty$ (atom at eigenvalue 1)
\item For $\rho = 1$: $m_n(1) \sim 1/\sqrt{2\pi n}$ (critical, from $\Pr(N=n)$ at the Poisson mode)
\item For $\rho > 1$: $m_n(\rho) \sim (\rho e^{1-\rho})^n/\sqrt{2\pi n}$ (exponential decay)
\end{itemize}

\subsection{Test Loss: Frobenius Norm Analysis}
\label{sec:proof-SStestloss}

We now prove the test loss theorem for shuffle-once.

\begin{theorem}[Shuffle-Once Test Loss]
\label{thm:SS-test-restated}
Let $A_{\rho d}$ denote the dynamics matrix. Let
$\phi_m(x, y) = \frac{1}{d} \sum_{i,j=0}^{\infty} x^i y^j \mathbb{E}[ \mathrm{Tr}((A_m)^i (A_m^\top)^j)]$
be the bivariate generating function.
Then, as $d\to\infty$:
\[
\lim_{d\to\infty} \phi_{\rho d}(x, y) = \frac{W_x W_y - 1}{(W_x+1)(W_y+1)} \left[ W_x W_y e^{-\rho (W_x W_y-1)} - 1 \right]^{-1}
\]
where $W_x = \rho^{-1} W_0(-\rho e^{-\rho} x)$ and $W_y = \rho^{-1} W_0(-\rho e^{-\rho} y)$.
\end{theorem}

\begin{corollary}
If $\rho=1$, then
\[
\lim_{d\to\infty} \; \frac{1}{d} \cdot \mathbb{E}[ \|A_d^n\|_F^2 ] = \frac{1}{2\sqrt{\pi n}} + O(n^{-1}).
\]
\end{corollary}

\begin{proof}
\textbf{Step 1: Set up the recursion.}
We introduce the two-sided generating function (resolvent):
\begin{align*}
\phi_A(x,y)
&=
\frac{1}{d} \mathbb{E} \,\mathrm{Tr}\left((A - x I)^{-1}(A^\top - y I)^{-1}\right)
= \sum_{n,k\ge 0} \frac{\frac{1}{d}\mathbb{E}\,\mathrm{Tr}(A^n (A^\top)^k)}{x^{n+1} y^{k+1}}.
\end{align*}
The diagonal terms $[(xy)^{-(k+1)}]\phi_A(x,y) = \frac{1}{d}\mathbb{E}\,\mathrm{Tr}(A^k (A^\top)^k) = \frac{1}{d}\mathbb{E}\|A^k\|_F^2$ give the test loss moments.

Let $R_x = (A - x I)^{-1}$ and $R_y = (A^\top - yI)^{-1}$.
Expanding the product using Sherman-Morrison and taking traces:
\begin{align*}
\mathrm{Tr}(R'_x R'_y)
&= \mathrm{Tr} (R_x R_y)
- \frac{ u^\top R_y R_x R_y u }{ u^\top R_y u}
- \frac{ u^\top R_x R_y R_x u }{ u^\top R_x u}
\\&\quad+ \frac{ u^\top R_x R_y u }{ xy \, u^\top R_x u\, u^\top R_y u}
+ \frac{ u^\top R_y R_x u\, u^\top R_x  R_y u}{ u^\top R_x u \, u^\top R_y u}.
\end{align*}

\textbf{Step 2: Concentration and expectation.}
By the same Lipschitz concentration argument, for $|x|, |y| > 1$:
\begin{align*}
\phi_{t+1}(x,y)
&= \phi_t(x,y)
- \frac{ \frac{\partial}{\partial y}\phi_t(x,y) }{ \phi_t(y)}
- \frac{ \frac{\partial}{\partial x} \phi_t(x,y) }{ \phi_t(x)}
\\&\quad+ \frac{ \phi_t(x,y) }{ xy \phi_t(x) \phi_t(y)}
+ \frac{ \phi_t(x,y)^2}{ \phi_t(x) \phi_t(y)}
+ o(1/d).
\end{align*}

\textbf{Step 3: Transform to PDEs.}
In the limit $d\to\infty$ with $t=\rho d$:
\begin{align}
\frac{\partial \phi_x}{\partial \rho}
&= -\frac{1}{x} -  \frac{1}{\phi_x} \frac{\partial \phi_x}{\partial x},
&\phi(0, x) &= \frac{1}{1-x}
\label{eq:phi1-frob}
\\
\frac{\partial \phi_{x,y}}{\partial \rho}
&= \frac{\phi_{x,y}}{\phi_x \phi_y}\left(
\frac{1}{xy} + \phi_{x,y}
\right)
- \frac{\frac{\partial}{\partial y}\phi_{x,y}}{\phi_y}
- \frac{\frac{\partial}{\partial x} \phi_{x,y}}{\phi_x},
&\phi(0, x, y) &= \frac{1}{1-x}\frac{1}{1-y}.
\label{eq:phi2-frob}
\end{align}

\textbf{Step 4: Solve using characteristics.}
By the method of characteristics with $C_x = x\phi_x$ and $C_y = y\phi_y$ as constants of motion:
\begin{align*}
\phi_x(\rho, x) &= \frac{-1}{x (1 + W_x)}, \\
\phi_{x,y}(\rho, x, y) &= \frac{W_x W_y - 1}{xy (W_x+1)(W_y+1)} \left[ W_x W_y e^{-\rho (W_x W_y-1)} - 1 \right]^{-1},
\end{align*}
where $W_x = \rho^{-1} W(-\rho e^{-\rho}/x)$ and $W_y = \rho^{-1} W(-\rho e^{-\rho}/y)$.

\textbf{Step 5: Extract diagonal coefficients via ACSV.}
For $\rho=1$, define $f(x,y)=\phi(1, 1/x, 1/y)/(xy)$ so that $a_{n,n}=[(xy)^n]f(x,y)$ gives the Frobenius norm moments.
Near the dominant singularity $(x,y)=(1,1)$:
\[
f(x,y)
=
\frac{1}{\sqrt 2\sqrt{(1-x)(1-y)}(\sqrt{1-x}+\sqrt{1-y})}
+ O((1-x)^{-1})
+ O((1-y)^{-1}).
\]

Using the ACSV (Analytic Combinatorics in Several Variables) diagonal extraction with $x=e^{-s/n}$, $y=e^{-t/n}$:
\begin{align*}
a_{n,n}
&\approx
\frac{n^{3/2}}{n^2} \frac{1}{\sqrt{2}(2\pi i)^2} \int_{c_1-i\infty}^{c_1+i\infty} \int_{c_2-i\infty}^{c_2+i\infty} \frac{e^{s+t}}{\sqrt{st}(\sqrt{s}+\sqrt{t})} ds \, dt
= \frac{1}{2\sqrt{\pi n}}.
\end{align*}

Thus $\frac{1}{d} \mathbb{E}\|A_d^n\|_F^2 \sim \frac{1}{2\sqrt{\pi n}}$.
The refined asymptotic expansion is:
\begin{equation}
\boxed{
\frac{1}{d}\,\mathbb{E}\|A_d^{\,n}\|_F^2
=
\frac{1}{2\sqrt{\pi}}\,n^{-1/2}
-\frac{1}{6\pi}\,n^{-1}
+O(n^{-3/2}).
}
\label{eq:frob-refined}
\end{equation}
\end{proof}

\subsection{Training Loss: Proof of Theorem~\ref{thm:shuffle-once-train}}
\label{sec:proof-trainloss}

\begin{theorem}[Shuffle-Once Training Loss]
In the proportional limit $m=\rho d$ with $\rho=1$, the training loss of shuffle-once Kaczmarz after $n$ epochs is
\[
\lim_{d\to\infty}\; \frac{1}{d} \cdot \mathbb{E}\bigl[\|X A_{d}^{n}\|_F^2 \bigr]
\;=\; \frac{1}{8\sqrt{\pi}}\; n^{-3/2}\,\Bigl(1+O(n^{-1})\Bigr).
\]
\end{theorem}

\begin{proof}
\textbf{Step 1: Set up the recursion.}
We track the cumulative outer-product matrix $S_t=\sum_{j=1}^{t}u_j u_j^{\top}$, noting that $X^\top X = S_m$.
Define the augmented resolvent:
\[
\psi_t(x,y)
=\frac{1}{d}\,\mathbb{E}\,\mathrm{Tr}\!\bigl[S_t(A_t-xI)^{-1}(A_t^{\top}-yI)^{-1}\bigr].
\]
The diagonal coefficient $[(xy)^{-(n+1)}]\psi_t(x,y) = \mathbb{E}\|X A_t^{n}\|_F^2/d$ is our target.

Using Sherman-Morrison with $R_x:=(A_t-xI)^{-1}$, $R_y:=(A_t^{\top}-yI)^{-1}$:
\begin{align}
\mathrm{Tr}\!\bigl[S_{t+1}R'_xR'_y\bigr]
= \mathrm{Tr}\!\bigl[(S_{t}+uu^\top)R'_xR'_y\bigr]
= \mathrm{Tr}\!\bigl[S_t R'_xR'_y\bigr] + u^{\top}R'_xR'_y u.
\label{eq:psi-update}
\end{align}
The second term equals $1/(xy)$ since $u$ is a null vector of the updated matrices.

\textbf{Step 2: Concentration and simplification.}
After careful expansion and application of concentration, many terms cancel, yielding:
\[
\psi_{t+1}(x,y)
= \psi_t(x,y)
- \frac{\partial_y \psi_t(x,y)}{\phi_t(y)}
- \frac{\partial_x \psi_t(x,y)}{\phi_t(x)}
+ \frac{\psi_t(x,y)}{xy \,\phi_t(x)\phi_t(y)}
+ \frac{\psi_t(x,y)\phi_t(x,y)}{\phi_t(x)\phi_t(y)}
+ \frac{1}{xy}
+ o(1/d).
\]

\textbf{Step 3: Transform to PDEs.}
In the limit $d\to\infty$, we have the system:
\begin{align}
\partial_\rho \psi_z
&=
- \frac{\psi_z}{z \phi_z}
- \frac{\partial_z \psi_z}{\phi_z}
- \frac{1}{z},
&\psi_z(0, z) &= 0
\label{eq:psi1}
\\
\partial_\rho \psi_{x,y} &=
- \frac{\partial_x \psi_{x,y}}{\phi_x}
- \frac{\partial_y \psi_{x,y}}{\phi_y}
+ \frac{\psi_{x,y}}{xy\phi_x \phi_y} + \frac{\psi_{x,y}\phi_{x,y}}{\phi_x \phi_y}
+ \frac{1}{xy},
&\psi(0,x,y) &= 0.
\label{eq:psi2}
\end{align}

\textbf{Step 4: Solve using characteristics.}
Take characteristics with $\frac{dx}{d\rho} = \frac{1}{\phi_x}$ and $\frac{dy}{d\rho} = \frac{1}{\phi_y}$.
Along these, $\alpha = x\phi_x$ and $\beta = y\phi_y$ are constants.

The single-variable weighted resolvent satisfies $\psi_z(\rho,z)=-\rho/z$.

For the two-variable case, define $F:=xy\,\phi_{x,y}$ and $A:=\frac{1}{\alpha}+\frac{1}{\beta}+\frac{1}{\alpha\beta}$.
The ODE for $F$ is:
\[
\frac{dF}{d\rho} = A\,F+\frac{1}{\alpha\beta}F^2,
\]
with solution:
\[
F(\rho)=\frac{\alpha\beta\,A\,e^{A\rho}}{A+1-e^{A\rho}}.
\]

\emph{The ratio trick.}
For $q:=\psi/\phi_{x,y}$, we find remarkably:
\[
\frac{dq}{d\rho} = \frac{1}{F},
\]
so with $q(0)=0$:
\[
\frac{\psi(\rho,x,y)}{\phi_{x,y}(\rho,x,y)}
= \int_0^\rho \frac{d\tau}{F(\tau)}
= \frac{1}{\alpha\beta\,A}\left(\frac{A+1}{A}\bigl(1-e^{-A\rho}\bigr)-\rho\right).
\]

\textbf{Step 5: Extract moments via ACSV.}
For $\rho=1$, near $(x,y)=(1,1)$:
\[
G(x,y) := \psi(1,1/x,1/y)/(xy)
\approx \frac{1}{\sqrt{2(1-x)}+\sqrt{2(1-y)}}.
\]
This is one order softer than the $\phi$ singularity, producing the extra factor of $1/n$.

Using diagonal extraction:
\[
\frac{1}{(2\pi i)^2}\iint \frac{e^{s+t}}{\sqrt{s}+\sqrt{t}}\,ds\,dt = \frac{1}{4\sqrt{2\pi}}.
\]

Therefore:
\[
[(xy)^{-(n+1)}]\psi_{1}(x,y)
= \frac{1}{\sqrt{2}} \cdot \frac{1}{4\sqrt{2\pi}} \cdot n^{-3/2}\left(1+O(n^{-1})\right)
= \frac{1}{8\sqrt{\pi}}n^{-3/2}\left(1+O(n^{-1})\right).
\]

The refined expansion is:
\begin{equation}
\boxed{
\lim_{d\to\infty}\frac{1}{d}\,\mathbb{E}\| X A_{d}^{n}\|_F^2
= \frac{1}{8\sqrt{\pi}}\;n^{-3/2}
+\frac{1}{12\pi}\;n^{-2}
+O(n^{-5/2}).
}
\label{eq:train-refined}
\end{equation}
\end{proof}

\begin{figure}[ht]
    \centering
    \includegraphics[width=0.75\linewidth]{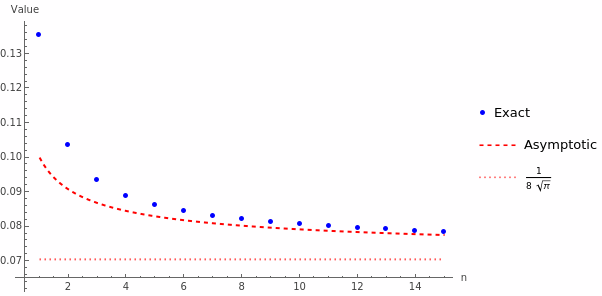}
    \caption{Training loss for shuffle-once at $\rho=1$, showing excellent agreement between empirical results and the theoretical prediction $\frac{1}{8\sqrt{\pi}}n^{-3/2}$.}
    \label{fig:single-shuffle-training}
\end{figure}

\subsection{Summary of Asymptotic Expansions}

For completeness, we collect all the refined asymptotic expansions for shuffle-once at $\rho=1$:

\begin{center}
\renewcommand{\arraystretch}{1.5}
\begin{tabular}{lcc}
\toprule
\textbf{Quantity} & \textbf{Leading Term} & \textbf{Second-Order Correction} \\
\midrule
Trace moment $m_n(1)$ & $\frac{1}{\sqrt{2\pi n}}$ & $1-\frac{1}{12n}+O(n^{-2})$ \\[0.3em]
Test loss $\frac{1}{d}\mathbb{E}\|A_d^n\|_F^2$ & $\frac{1}{2\sqrt{\pi n}}$ & $1-\frac{1}{3\sqrt{\pi n}}+O(n^{-1})$ \\[0.3em]
Training loss $\frac{1}{d}\mathbb{E}\|XA_d^n\|_F^2$ & $\frac{1}{8\sqrt{\pi}}n^{-3/2}$ & $1+\frac{2}{3\sqrt{\pi n}}+O(n^{-1})$ \\
\bottomrule
\end{tabular}
\end{center}

The factor of 4 between test and training loss leading constants ($\frac{1}{2\sqrt{\pi}}$ vs $\frac{1}{8\sqrt{\pi}}$) reflects the faster convergence on training examples that are directly used for optimization.

\section{Flip-Flop Analysis}
\label{app:flipflop}

This appendix provides detailed proofs for the flip-flop sampling scheme, where odd epochs process rows in a fixed order and even epochs process in reverse order.

\subsection{Test Loss: Proof of Theorem~\ref{thm:flipflop-test}}
\label{sec:proof-FFtestloss}

\begin{theorem}[Flip-Flop Test Loss]
Let $A_{\rho d}$ be the one-epoch dynamics matrix with $m=\rho d$, and let $n$ be even. Then,
\[
\lim_{d\to\infty} \; \frac{1}{d} \; \mathbb{E}[ \|(A_{\rho d}^\top A_{\rho d})^{n/2}\|_F^2 ]
= \frac{1}{\sqrt{2 \pi n}} + O\Big(\frac{1}{n}\Big)
\qquad\text{if $\rho=1$ and $n\to\infty$}.
\]
\end{theorem}

\begin{proof}
Under the flip-flop schedule, two epochs multiply by the same $B_m := A_m^{\top}A_m$.
Thus after $n$ (even) epochs the squared test loss is
\[
\frac{1}{d}\,\mathbb{E}\bigl\|(A_m^{\top}A_m)^{n/2}\bigr\|_F^2
=\frac{1}{d}\,\mathbb{E}\,\mathrm{Tr}\!\bigl[(A_m^{\top}A_m)^n\bigr].
\]
We analyze these moments via the Stieltjes transform of $B_t:=A_t^{\top}A_t$.

\paragraph{Step 1: Set up the recursion.}
Let $u=u_{t+1}$ be the next unit row direction:
\[
A_{t+1}=(I-uu^\top)A_t,
\qquad
B_{t+1}=A_{t+1}^{\top}A_{t+1}
=A_t^{\top}(I-uu^\top)A_t
=B_t-vv^\top,
\]
where $v:=A_t^{\top}u$.
For $|\lambda|>1$, denote $R:=(B_t-\lambda I)^{-1}$ and $R':=(B_{t+1}-\lambda I)^{-1}$.
By Sherman-Morrison:
\[
R' = R+\frac{Rv v^\top R}{1-v^\top Rv}.
\]
Taking traces:
\begin{equation}
\mathrm{Tr}\, R' - \mathrm{Tr}\, R = \frac{v^\top R^2 v}{1 - v^\top R v}.
\label{eq:FF-test-one-step}
\end{equation}

\paragraph{Step 2: Concentration argument.}
Since $v=A_t^{\top}u$ with $u$ uniform on the sphere and $\|A_t\|\le 1$, the quadratic forms $f(u):=v^\top R^k v = u^\top (A_t R^k A_t^{\top})u$ are $2\|R^k\|$-Lipschitz.
By concentration on the sphere:
\[
v^\top R^k v
=\frac{1}{d}\mathrm{Tr}(A_t R^k A_t^{\top})+o_{\mathbb{P}}(1)
=\frac{1}{d}\mathrm{Tr}(R^k B_t)+o_{\mathbb{P}}(1).
\]
Using the identity $RB_t=I+\lambda R$:
\[
v^\top R v = 1+\lambda\frac{1}{d}\mathrm{Tr}\, R+o(1),
\qquad
v^\top R^2 v = \frac{1}{d}\mathrm{Tr}\, R+\lambda\frac{1}{d}\mathrm{Tr}\, R^2+o(1).
\]
Substituting into~\eqref{eq:FF-test-one-step}:
\[
\mathbb{E}[\mathrm{Tr}\, R' - \mathrm{Tr}\, R]
= -\frac{1}{\lambda}
- \frac{\mathbb{E}\,\mathrm{Tr}\, R^2}{\mathbb{E}\,\mathrm{Tr}\, R} + o(1).
\]

\paragraph{Step 3: Transport equation.}
Define the Stieltjes transform $s(\rho,\lambda):=\lim_{d\to\infty}\frac{1}{d}\mathbb{E}\,\mathrm{Tr}(B_{\rho d}-\lambda I)^{-1}$.
The transport equation is:
\begin{equation}
\partial_\rho s + \frac{1}{s}\,\partial_\lambda s = -\frac{1}{\lambda},
\qquad
s(0,\lambda)=\frac{1}{1-\lambda}.
\label{eq:PDE-FF-test}
\end{equation}

\paragraph{Step 4: Connection to shuffle-once trace moments.}
Equation~\eqref{eq:PDE-FF-test} is \emph{identical} to the scalar PDE for shuffle-once trace moments (Section~\ref{sec:warmup} of Appendix~\ref{app:stieltjes}).
Since the PDE and initial conditions are the same, the solutions coincide.
Hence the flip-flop test loss moments $b_n(\rho) := \lim_{d\to\infty}\tfrac{1}{d}\mathbb{E}\,\mathrm{Tr}(B_{\rho d}^n)$ equal the shuffle-once trace moments $m_n(\rho)$.

At criticality $\rho=1$:
\begin{equation}
\boxed{
\lim_{d\to\infty}\frac{1}{d}\,\mathbb{E}\bigl\|(A_d^{\top}A_d)^{n/2}\bigr\|_F^2
=\frac{1}{\sqrt{2\pi n}}\Bigl(1+O(n^{-1})\Bigr).
}
\label{eq:FF-test-main}
\end{equation}

The refined asymptotic expansion from Stirling's series:
\begin{equation}
\boxed{
\frac{1}{d}\,\mathbb{E}\bigl\|(A_d^{\top}A_d)^{n/2}\bigr\|_F^2
=\frac{1}{\sqrt{2\pi n}}\left(
1-\frac{1}{12n}+\frac{1}{288n^2}+O(n^{-3})
\right).
}
\label{eq:FF-test-refined}
\end{equation}
\end{proof}

\subsection{Training Loss: Proof of Theorem~\ref{thm:flipflop-train}}
\label{sec:proof-FFtrainloss}

\begin{theorem}[Flip-Flop Training Loss]
In the proportional regime $m=\rho d$ with $\rho=1$, the flip-flop training loss after $n$ epochs (with $n$ even) is
\[
\lim_{d\to\infty}\frac{1}{d}\,\mathbb{E}\bigl\|X\,(A_d^{\!\top}A_d)^{n/2}\bigr\|_F^2
\;=\;
\frac{\sqrt{2}}{3\sqrt{\pi}}\,n^{-3/2}\Bigl(1+O(n^{-1})\Bigr).
\]
\end{theorem}

\begin{remark}
The flip-flop training loss constant $\frac{\sqrt{2}}{3\sqrt{\pi}}$ is a factor $\frac{8\sqrt{2}}{3} \approx 3.77$ worse than the shuffle-once constant $\frac{1}{8\sqrt{\pi}}$.
The ``flop'' in flip-flop is actively harmful for training loss.
\end{remark}

\begin{proof}
Two flip-flop epochs multiply by $B_m:=A_m^{\!\top}A_m$, so
\[
\frac{1}{d}\,\mathbb{E}\|X(B_m)^{n/2}\|_F^2
=\frac{1}{d}\,\mathbb{E}\,\mathrm{Tr}\!\bigl[S_m B_m^{n}\bigr],
\]
where $S_m:=X^\top X=\sum_{j=1}^m u_ju_j^\top$.
We study the resolvent-weighted trace:
\[
\gamma_t(\lambda)\ :=\ \frac{1}{d}\,\mathbb{E}\,\mathrm{Tr}\!\bigl[S_t (B_t-\lambda I)^{-1}\bigr].
\]

\paragraph{Step 1: Resolvent update.}
With $B_{t+1}=B_t - v v^\top$ where $v=A_t^{\top}u$, Sherman-Morrison gives:
\[
R_{t+1} = R_t + \frac{R_tv v^\top R_t}{1-v^\top R_t v}.
\]
Key identities:
\[
A_t R_t A_t^\top = I + \lambda\,(A_tA_t^\top - \lambda I)^{-1},
\qquad
R_t B_t = B_t R_t = I + \lambda R_t.
\]

\paragraph{Step 2: Transport system.}
Introduce three auxiliary functions:
\[
\alpha_t(\lambda):=\frac{1}{d}\,\mathbb{E}\,\mathrm{Tr}\, R_t,\quad
\beta_t(\lambda):=\frac{1}{d}\,\mathbb{E}\,\mathrm{Tr}(A_tR_t),\quad
\gamma_t(\lambda):=\frac{1}{d}\,\mathbb{E}\,\mathrm{Tr}(S_tR_t).
\]

In the limit $d\to\infty$ with $\rho=t/d$, we obtain the system:
\begin{align}
\partial_\rho \alpha &= -\frac{1}{\lambda} - \frac{\partial_\lambda \alpha}{\alpha},
&\alpha(0,\lambda)&=\frac{1}{1-\lambda},
\label{eq:PDE-alpha}
\\
\partial_\rho \beta &= -\frac{1}{\alpha}\,\partial_\lambda \beta,
&\beta(0,\lambda)&=\frac{1}{1-\lambda},
\label{eq:PDE-beta}
\\
\partial_\rho \gamma &= \alpha - \frac{\gamma + \lambda\,\partial_\lambda\gamma}{\lambda\,\alpha}
- \frac{\beta^2}{\lambda\,\alpha},
&\gamma(0,\lambda)&=0.
\label{eq:PDE-gamma}
\end{align}

\paragraph{Step 3: Solution by characteristics.}
Equation~\eqref{eq:PDE-alpha} is solved by characteristics with conserved quantity $C:=\lambda\,\alpha(\rho,\lambda)$:
\[
\lambda(\rho)=\lambda_0\,e^{\rho/C},\qquad
\alpha(\rho,\lambda)=\frac{C}{\lambda},\qquad
C=\frac{\lambda_0}{1-\lambda_0}.
\]
In closed form (principal branch $W_0$):
\begin{equation}
\alpha(\rho,\lambda)
=-\frac{1}{\bigl(1+\rho^{-1} W_0(-\rho e^{-\rho}/\lambda)\bigr)\,\lambda}.
\label{eq:alpha-Lambert}
\end{equation}

For $\beta$, equation~\eqref{eq:PDE-beta} implies $\beta$ is constant along characteristics:
\[
\beta(\rho,\lambda)=\beta(0,\lambda_0)=\frac{1}{1-\lambda_0}=1+C
\quad\Longrightarrow\quad
\beta(\rho,\lambda)=1+\lambda\,\alpha(\rho,\lambda).
\]

For $\gamma$, equation~\eqref{eq:PDE-gamma} becomes a linear ODE along characteristics:
\[
\frac{d}{d\rho}\gamma(\rho)+\frac{1}{C}\,\gamma(\rho)
=\alpha(\rho)-\frac{\beta^2}{C}.
\]
With $\gamma(0)=0$:
\begin{equation}
\gamma(\rho,\lambda)
=\Bigl[(1+C)^2+(1+C)\,\rho\Bigr]\,e^{-\rho/C}-(1+C)^2,
\qquad C=\lambda\,\alpha(\rho,\lambda).
\label{eq:gamma-solved}
\end{equation}

\paragraph{Step 4: Closed-form generating function.}
The training loss coefficients $a_n(\rho) := \lim_{d\to\infty}\frac{1}{d}\,\mathbb{E}\,\mathrm{Tr}[S_m B_m^{\,n}]$ are encoded by:
\[
\mathcal{G}_\rho(z) := \sum_{n\ge0} a_n(\rho)\,z^n
= -\frac{1}{z}\,\gamma\!\left(\rho,\frac{1}{z}\right).
\]

Setting $W := W_0(-\rho e^{-\rho} z)$:
\begin{equation}
\boxed{
\mathcal{G}_\rho(z)
= \frac{ W^2 + \rho z\bigl(\rho^2+(1+\rho)W\bigr)}{z\,(\rho+W)^2},
\qquad
W=W_0\!\bigl(-\rho e^{-\rho} z\bigr).
}
\label{eq:G-closed}
\end{equation}

\paragraph{Step 5: Edge asymptotics at $\rho=1$.}
As $z\uparrow 1$, using $W_0(-e^{-1}+u)=-1+\sqrt{2u}-\frac{2}{3}u+O(u^{3/2})$:
\begin{equation}
\mathcal{G}_1(z)
= \frac{3}{2} - \frac{2\sqrt{2}}{3}\,\sqrt{1-z} + O(1-z).
\label{eq:G-edge}
\end{equation}
By the Flajolet-Odlyzko transfer theorem, $[z^n](1-z)^{1/2}\sim -\frac{1}{2\sqrt{\pi}}\,n^{-3/2}$, hence:
\begin{equation}
\boxed{
a_n(1) = \frac{\sqrt{2}}{3\sqrt{\pi}}\, n^{-3/2}\Bigl(1+O(n^{-1})\Bigr).
}
\label{eq:an-asymp-closure}
\end{equation}

The refined expansion:
\begin{equation}
\boxed{
a_n(1)=\frac{\sqrt{2}}{3\sqrt{\pi}}\,n^{-3/2}
+\frac{203\sqrt{2}}{360\sqrt{\pi}}\,n^{-5/2}
+O(n^{-7/2}).
}
\label{eq:an-asymp-refined}
\end{equation}
\end{proof}

\subsection{Exact Coefficient Formulas}

The generating function~\eqref{eq:G-closed} yields closed-form expressions for the coefficients.

\paragraph{First few coefficients.}
For general $\rho>0$:
\begin{align*}
a_0(\rho)&=\rho,\\
a_1(\rho)&=e^{-\rho}(\rho-1)+e^{-2\rho},\\
a_2(\rho)&=\bigl(2\rho+(\rho^2-2)e^{\rho}+2\bigr)e^{-3\rho},\\
a_3(\rho)&=\tfrac{1}{2}\Bigl(8\rho^2+12\rho+\bigl(3\rho^3+\rho^2-6\rho-6\bigr)e^{\rho}+6\Bigr)e^{-4\rho}.
\end{align*}

\paragraph{Probabilistic representation.}
Let $\lambda:=\rho n$ and $\mu:=\rho(n+1)$. For $n\ge1$, if $J\sim\mathrm{Poisson}(\lambda)$ and $K\sim\mathrm{Poisson}(\mu)$:
\[
a_n(\rho)=\frac{1}{n}\,\mathbb{E}\!\big[(n-J)(\rho-n+J)\,\mathbf{1}_{\{J\le n-1\}}\big]
+\frac{1}{n+1}\,\mathbb{E}\!\big[(n+1-K)(n-K)\,\mathbf{1}_{\{K\le n-1\}}\big].
\]

\paragraph{Poisson CDF form.}
Writing the Poisson CDF $Q_m(x):=\Pr\{\mathrm{Poisson}(x)\le m\} = e^{-x}\sum_{j=0}^{m} \frac{x^j}{j!}$:
\[
\begin{aligned}
a_n(\rho)=&
\frac{1}{n}\Big[n(\rho-n)\,Q_{n-1}(\lambda)
+ (2n-\rho-1)\,\lambda\,Q_{n-2}(\lambda)
- \lambda^{2}\,Q_{n-3}(\lambda)\Big]\\
&+\frac{1}{n+1}\Big[n(n+1)\,Q_{n-1}(\mu)
- 2n\,\mu\,Q_{n-2}(\mu)
+ \mu^{2}\,Q_{n-3}(\mu)\Big],
\end{aligned}
\]
where $\lambda=\rho n$ and $\mu=\rho(n+1)$.

\subsection{Comparison of Shuffling Methods}

We summarize the asymptotic constants for all three sampling methods at $\rho=1$:

\begin{center}
\renewcommand{\arraystretch}{1.5}
\begin{tabular}{lccc}
\toprule
\textbf{Method} & \textbf{Test Loss} & \textbf{Training Loss} & \textbf{Test Loss Ratio} \\
\midrule
Shuffle-Once & $\frac{1}{2\sqrt{\pi n}}$ & $\frac{1}{8\sqrt{\pi}}n^{-3/2}$ & 1 (baseline) \\[0.3em]
IID Sampling & $\frac{1}{\sqrt{2\pi n}}$ & $\frac{1}{2\sqrt{2\pi}}n^{-3/2}$ & $\sqrt{2} \approx 1.41$ \\[0.3em]
Flip-Flop & $\frac{1}{\sqrt{2\pi n}}$ & $\frac{\sqrt{2}}{3\sqrt{\pi}}n^{-3/2}$ & $\sqrt{2} \approx 1.41$ \\
\bottomrule
\end{tabular}
\end{center}

Key observations:
\begin{itemize}
\item Shuffle-once achieves a $\sqrt{2}$ improvement in test loss over both IID and flip-flop.
\item For training loss, shuffle-once is $2\sqrt{2}\approx 2.83\times$ better than IID and $\approx 3.77\times$ better than flip-flop.
\item Flip-flop has the same test loss constant as IID, but worse training loss---the reverse pass provides no benefit and actively harms training convergence.
\end{itemize}

\section{Learning Rate Generalizations}
\label{app:lr}

This appendix extends the main results to general learning rates $\eta \in (0, 2)$.
The update rule becomes:
\[
A_{t+1}=(I-\eta\,u_{t+1}u_{t+1}^\top)\,A_t,\qquad \|u_{t+1}\|=1.
\]
For $\eta \in (0, 2)$, we have $\|I-\eta uu^\top\|\le 1$, ensuring $\|A_t\|\le 1$ for all $t$.

\subsection{Trace Moments with Learning Rate}
\label{sec:warmup-eta}

We first generalize the trace moment analysis.
Define the Stieltjes transform:
\begin{align}
s_t(\lambda;\eta)
=\frac{1}{d}\,\mathbb{E}\,\mathrm{Tr}\big((A_t-\lambda I)^{-1}\big)
=-\sum_{n\ge 0}\frac{1}{d}\,\mathbb{E}\,\mathrm{Tr}(A_t^n)\,\lambda^{-n-1}.
\label{eq:stieltjes-simple-eta}
\end{align}

\begin{theorem}[Trace Moments with Learning Rate]
For $\rho=1$ and $\eta \in (0,2)$:
\begin{equation}
m_n(\eta,1)=\frac{\sqrt{a(\eta)}}{\sqrt{\pi}}\;n^{-1/2}\Big(1+O(n^{-1/2})\Big),
\qquad a(\eta)=\frac{1-\eta/2}{\eta}.
\label{eq:crit-asymp-eta}
\end{equation}
For $\eta=1$, this recovers $m_n(1,1)\sim (2\pi n)^{-1/2}$.
\end{theorem}

\begin{proof}
\textbf{Step 1: One-step update with $\eta$.}
Write $R=(A_t-\lambda I)^{-1}$. Sherman-Morrison gives:
\begin{equation}
\mathrm{Tr}\big((A_{t+1}-\lambda I)^{-1}\big)-\mathrm{Tr}(R)
=\frac{\eta\,u^\top R u+\eta\lambda\,u^\top R^2 u}{\,1-\eta-\eta\lambda\,u^\top R u\,}.
\label{eq:onestep-eta}
\end{equation}

\textbf{Step 2: Concentration and PDE.}
Taking expectations and passing to the scaling limit $\rho=t/d$:
\begin{equation}
\partial_\rho s(\rho,\lambda;\eta)
=\frac{\,\eta\,s(\rho,\lambda;\eta)\;+\;\eta\lambda\,\partial_\lambda s(\rho,\lambda;\eta)\,}
{\,1-\eta-\eta\lambda\,s(\rho,\lambda;\eta)\,},
\qquad
s(0,\lambda;\eta)=\frac{1}{1-\lambda}.
\label{eq:pde-eta}
\end{equation}
For $\eta=1$ this reduces to the standard transport equation.

\textbf{Step 3: Method of characteristics.}
The product $C_*:=\lambda(\rho)\,s(\rho)$ is conserved along characteristics:
\[
\frac{d}{d\rho}\big(\lambda s\big)=0
\qquad\Rightarrow\qquad \lambda(\rho)\,s(\rho)=C_* \text{ (constant).}
\]
The implicit solution is:
\begin{equation}
\lambda=\frac{C_*}{1+C_*}\,\exp\!\Big(-\frac{\eta\rho}{\,1-\eta-\eta C_*\,}\Big),
\qquad
s(\rho,\lambda;\eta)=\frac{C_*}{\lambda}.
\label{eq:implicit-s-eta}
\end{equation}

\textbf{Step 4: Exact moments via Lagrange-Bürmann.}
Setting $u=1+h:=-C_*$, the Lagrange-Bürmann formula gives for $n\ge1$:
\begin{equation}
m_n(\eta,\rho)=\frac{1}{n}\,[h^{\,n-1}]\,(1+h)^n\,\exp\!\Big(-\frac{n\,\eta\rho}{\,1+\eta h\,}\Big).
\label{eq:LB-mn-eta}
\end{equation}
The first few moments are:
\[
\begin{aligned}
m_1(\eta,\rho)&=e^{-\eta\rho},\\
m_2(\eta,\rho)&=e^{-2\eta\rho}\big(1+\eta^2\rho\big),\\
m_3(\eta,\rho)&=e^{-3\eta\rho}\Big(1+3\eta^2\rho-\eta^3\rho+\tfrac{3}{2}\eta^4\rho^2\Big).
\end{aligned}
\]

\textbf{Step 5: Critical asymptotics.}
For $\rho=1$, as $u\to\infty$:
\[
\Psi(u;\eta)=1-\frac{a(\eta)}{u^2}+O(u^{-3}),
\qquad
a(\eta)=\frac{1-\eta/2}{\eta}.
\]
Hence $M(z;\eta)\sim \sqrt{a(\eta)}\,(1-z)^{-1/2}$ as $z\uparrow 1$, giving~\eqref{eq:crit-asymp-eta}.
\end{proof}

\begin{proposition}[Monotonicity in Learning Rate]
For fixed $n$ and $\rho$, $m_n(\eta,\rho)$ is strictly decreasing in $\eta \in (0, 2)$.
\end{proposition}

\begin{proof}
Differentiating~\eqref{eq:LB-mn-eta}:
\[
\partial_\eta m_n(\eta,\rho)
=-\,\rho\,[h^{\,n-1}]\,\frac{(1+h)^n}{(1+\eta h)^2}\,
\exp\!\Big(-\frac{n\,\eta\rho}{1+\eta h}\Big)\;<\;0.
\]
The coefficient is strictly positive since the integrand is a positive real-analytic function.
\end{proof}

\subsection{Shuffle-Once Test Loss with Learning Rate}
\label{sec:proof-SStestloss-eta}

We generalize the Frobenius norm analysis to $\eta \in (0,2)$.
Define the two-resolvent function:
\[
\phi(\rho;x,y;\eta)
:=\lim_{d\to\infty}\frac{1}{d}\,\mathbb{E}\,\mathrm{Tr}\!\Big((A_{\rho d}-xI)^{-1}(A_{\rho d}^\top-yI)^{-1}\Big).
\]

\begin{theorem}[Shuffle-Once Test Loss with Learning Rate]
For $\rho=1$ and fixed $\eta\in(0,2)$, as $n\to\infty$:
\[
c_{n}(\eta,1) \sim \sqrt{\frac{\,\tfrac{1}{\eta}-\tfrac{1}{2}\,}{2\pi}}\;\frac{1}{\sqrt{n}}.
\]
\end{theorem}

\begin{proof}
\textbf{Step 1: Two-resolvent update.}
Write $R_t=(A_t-xI)^{-1}$ and $S_t=(A_t^\top-yI)^{-1}$.
Sherman-Morrison gives update formulas for $\mathrm{Tr}(R_{t+1}S_{t+1})$.

\textbf{Step 2: Transport PDEs.}
Define $\alpha:=x\,\phi_x$, $\beta:=y\,\phi_y$, $\Delta_x:=1-\eta-\eta\alpha$, $\Delta_y:=1-\eta-\eta\beta$.
Let $\tau=\eta(2-\eta)$ and $F:=xy\,\phi$. The key equation is:
\begin{equation}
\frac{dF}{d\rho} = \kappa F + \gamma F^2,
\label{eq:phi-deriv-eta}
\end{equation}
where $\kappa:=\frac{\eta^2(1+\alpha+\beta)}{\Delta_x\Delta_y}$ and $\gamma:=\frac{\eta^2}{\Delta_x\Delta_y}$.

\textbf{Step 3: Solution along characteristics.}
With initial condition $F(0)=\alpha\beta$:
\[
F(\rho;\alpha,\beta)
=\frac{\kappa\,\alpha\beta\,e^{\kappa\rho}}{\kappa+\gamma\,\alpha\beta\,(1-e^{\kappa\rho})}.
\]
The characteristic equations:
\begin{align}
x=\frac{\alpha}{1+\alpha}\,\exp\!\Big(-\frac{\eta\rho}{\Delta_x}\Big),
\qquad
y=\frac{\beta}{1+\beta}\,\exp\!\Big(-\frac{\eta\rho}{\Delta_y}\Big).
\label{eq:C1}
\end{align}

\textbf{Step 4: Coefficient extraction.}
For $\eta\neq 1$, the two-variable Lagrange-Bürmann formula gives for $n\ge 1$:
\begin{equation*}
c_n(\eta,\rho)
=\frac{1}{n^2}\,[\alpha^{n-1}\beta^{n-1}]\,
\partial_\alpha\partial_\beta\!\Big(F(\rho;\alpha,\beta)\,\Psi(\alpha)^n\,\Psi(\beta)^n\Big).
\end{equation*}

\textbf{Step 5: Diagonal asymptotics.}
As $(x,y)\to(1,1)$, the characteristic inversion yields a square-root cusp, giving:
\begin{equation}
c_n(\eta,1)\;=\;\sqrt{\frac{\tfrac{1}{\eta}-\tfrac{1}{2}}{2\pi}}\;n^{-1/2}\Bigl(1+O(n^{-1/2})\Bigr).
\label{eq:crit-asymp}
\end{equation}
\end{proof}

The first few test loss coefficients (at general $\rho$) are
$c_0=1$,
$c_1=e^{-(2-\eta)\eta\rho}$,
and
$c_2=e^{-(4-\eta)\eta\rho + \eta^2\rho}
+ e^{-(4-\eta)\eta\rho} (2-\eta)^2 \eta^2 \rho (1+\eta^2\rho)$.

\begin{figure}[ht]
   \centering
   \includegraphics[width=0.7\textwidth]{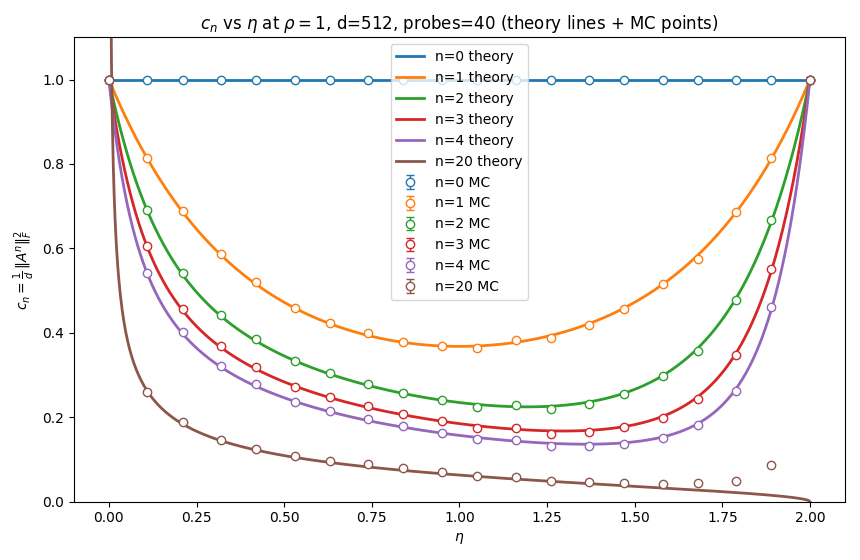}
   \caption{Single-shuffle test loss with learning rate $\eta\in[0,2]$ at $\rho=1$.
   Interestingly, the optimal $\eta$ is not at 1 when $n\ge 2$.
   For $n=20$ we use the asymptotic formula derived above. Note: Legend label $\alpha$ corresponds to step size $\eta$ used in the text.}
   \label{fig:single-shuffle-test-loss-lr}
\end{figure}

\subsection{Flip-Flop Test Loss with Learning Rate}
\label{sec:FF-eta-proof}

For flip-flop, the key quantity is:
let $\tau:=\eta(2-\eta)\in(0,1]$ for $\eta\in(0,2)$.
This arises because $(I-\eta uu^\top)^2 = I - \tau uu^\top$.

\begin{theorem}[Flip-Flop Test Loss with Learning Rate]
For $\rho=1$ and $\eta\in(0,2)$:
\begin{equation}
b_n(1,\eta)\sim
\sqrt{\frac{\,\tfrac{1}{\tau}-\tfrac{1}{2}\,}{\pi}}\;\frac{1}{\sqrt{n}},
\qquad \tau=\eta(2-\eta).
\label{eq:bn-asymp-critical}
\end{equation}
At $\eta=1$ ($\tau=1$), this gives $b_n(1,1)\sim (2\pi n)^{-1/2}$.
\end{theorem}

\begin{proof}
\textbf{Step 1: Modified dynamics.}
With $v:=A_t^\top u$:
\[
B_{t+1}
=A_t^\top(I-\eta uu^\top)^2A_t
=B_t-\tau\,vv^\top,
\quad \tau:=\eta(2-\eta).
\]

\textbf{Step 2: Transport PDE.}
Setting $y:=1+\lambda s$:
\begin{equation}
\partial_\rho y
=
\frac{\tau\,\lambda}{1-\tau y}\,\partial_\lambda y,
\qquad
y(0,\lambda)=\frac{1}{1-\lambda}.
\label{eq:PDE-y}
\end{equation}

\textbf{Step 3: Implicit solution.}
\begin{equation}
\lambda
=\Bigl(1-\frac{1}{y}\Bigr)\exp\!\Bigl(-\frac{\tau\,\rho}{1-\tau y}\Bigr),
\qquad
y=1+\lambda s(\rho,\lambda).
\label{eq:implicit-y}
\end{equation}

\textbf{Step 4: Exact coefficient formula.}
Setting $w:=1-\Phi$ (where $\Phi(z)=\sum_{n\ge 0}b_n z^n$), the Lagrange-Bürmann theorem gives:
\begin{equation}
\boxed{
b_n(\rho,\eta)
=\frac{(-1)^{\,n+1}}{n}\,[u^{\,n-1}]\,
(1-u)^n\exp\!\Bigl(-\frac{n\,\tau\,\rho}{\,1-\tau u\,}\Bigr),
\qquad n\ge 1.
}
\label{eq:bn-coeff}
\end{equation}

\textbf{Step 5: Verification.}
For small $n$:
\begin{align*}
b_1(\rho,\eta)&=e^{-\tau\rho}, \\
b_2(\rho,\eta)&=(1+\tau^2\rho)\,e^{-2\tau\rho}, \\
b_3(\rho,\eta)&=\Bigl(1+3\tau^2\rho-\tau^3\rho+\tfrac{3}{2}\tau^4\rho^2\Bigr)e^{-3\tau\rho}.
\end{align*}
At $\eta=1$ ($\tau=1$, $\rho=1$): $b_1=e^{-1}$, $b_2=2e^{-2}$, $b_3=\tfrac{9}{2}e^{-3}$, matching the exact formula $b_n(1,1)=e^{-n}n^n/n!$.

\begin{figure}[ht]
   \centering
   \includegraphics[width=0.7\textwidth]{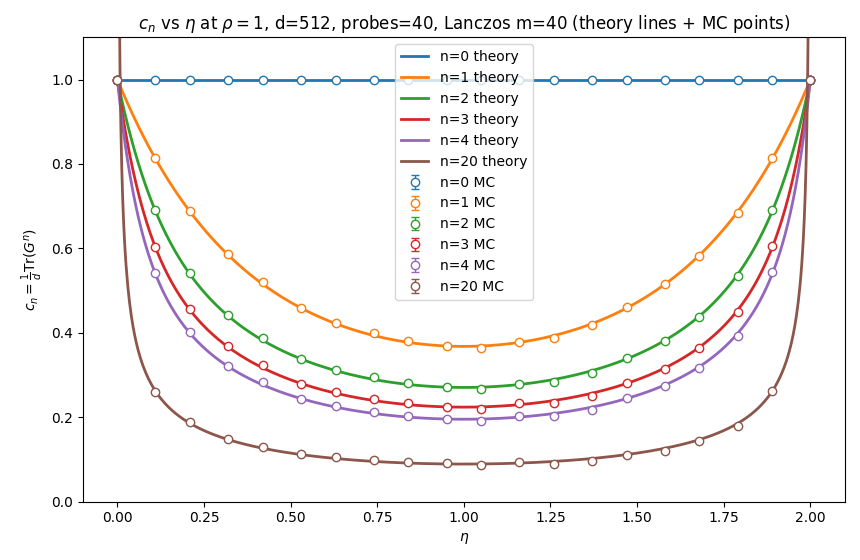}
   \caption{Flip-flop test loss with learning rate $\eta\in[0,2]$ at $\rho=1$.}
   \label{fig:flip-flop-test-loss-lr}
\end{figure}

\textbf{Step 6: Asymptotics.}
For $\rho=1$, expanding~\eqref{eq:implicit-y} as $y\to\infty$:
\[
z=\frac{1}{\lambda}
=1+\Bigl(\frac{1}{2}-\frac{1}{\tau}\Bigr)t^2+O(t^3),
\quad\Longrightarrow\quad
t=\sqrt{\frac{1-z}{\,\frac{1}{\tau}-\frac{1}{2}\,}}.
\]
The singular part gives~\eqref{eq:bn-asymp-critical}.
\end{proof}

\subsection{Summary: Learning Rate Dependence}

For all three quantities at $\rho=1$, the dependence on $\eta$ appears through a common factor:

\begin{center}
\renewcommand{\arraystretch}{1.5}
\begin{tabular}{lcc}
\toprule
\textbf{Quantity} & \textbf{Asymptotic Form} & \textbf{Coefficient} \\
\midrule
Trace moment & $\sqrt{\frac{a(\eta)}{\pi}}\,n^{-1/2}$ & $a(\eta) = \frac{1-\eta/2}{\eta}$ \\[0.5em]
Shuffle-once test & $\sqrt{\frac{a(\eta)}{2\pi}}\,n^{-1/2}$ & $a(\eta) = \frac{1}{\eta}-\frac{1}{2}$ \\[0.5em]
Flip-flop test & $\sqrt{\frac{a(\tau)}{\pi}}\,n^{-1/2}$ & $a(\tau) = \frac{1}{\tau}-\frac{1}{2}$, $\tau=\eta(2-\eta)$ \\
\bottomrule
\end{tabular}
\end{center}

Key observations:
\begin{itemize}
\item All convergence rates remain $n^{-1/2}$ regardless of $\eta$---only the leading constant changes.
\item The trace moment coefficient $a(\eta) = (1-\eta/2)/\eta$ is strictly decreasing in $\eta$ for $\eta \in (0, 2)$.
\item For flip-flop, the effective learning rate is $\tau = \eta(2-\eta)$, which is maximized at $\eta=1$.
\item At $\eta=1$: all formulas reduce to the standard Kaczmarz results derived in the main text.
\end{itemize}

The monotonicity in $\eta$ has practical implications: larger learning rates (up to $\eta < 2$ for stability) lead to faster convergence, but with diminishing returns as $\eta \to 2$.

\end{document}